\pdfoutput=1

\documentclass[11pt]{article}

\usepackage{amsmath}
\usepackage{amssymb}
\usepackage{epsfig}
\usepackage{etoolbox}
\usepackage{fullpage}
\usepackage{ifthen}
\usepackage{theorem}
\usepackage{verbatim}
\usepackage[pagebackref,colorlinks=true,urlcolor=black,linkcolor=black,citecolor=black,pdfstartview=]{hyperref}

\newtheorem{theorem}{Theorem}[section]

\newtheorem{corollary}[theorem]{Corollary}
\newtheorem{fact}[theorem]{Fact}
\newtheorem{lemma}[theorem]{Lemma}

\newtheorem{proposition}[theorem]{Proposition}
{\theorembodyfont{\rmfamily}\newtheorem{definition}[theorem]{Definition}}
{\theorembodyfont{\rmfamily}}
{\theorembodyfont{\rmfamily}\newtheorem{remark}[theorem]{Remark}}
\theoremstyle{break}
{\theorembodyfont{\rmfamily}}

\newcommand{\A}{\mathcal{A}}

\newcommand{\F}{\mathcal{F}}
\newcommand{\G}{\mathcal{G}}

\newcommand{\cH}{\mathcal{H}}

\newcommand{\cP}{\mathcal{P}}
\newcommand{\RR}{\mathbb{R}}

\newcommand{\Probl}{\Pi}
\newcommand{\dist}{\mathcal{D}}
\newcommand{\distr}{\mathcal{D}}
\newcommand{\cost}{\text{\sc cost}}
\newcommand{\prob}[2][]{\text{\bf Pr}\ifthenelse{\not\equal{}{#1}}{_{#1}}{}\!\left[#2\right]}
\newcommand{\expect}[2][]{\text{\bf E}\ifthenelse{\not\equal{}{#1}}{_{#1}}{}\!\left[#2\right]}
\newcommand{\sse}{\subseteq}
\newcommand{\eps}{\epsilon}
\newcommand{\ed}{(\eps,\delta)}
\newcommand{\kb}{(\kappa,\beta)}
\def\a{\alpha}
\def\b{\beta}

\def\df{\nabla f}
\def\grad{{\nabla}}
\def\la{\langle}
\def\ra{\rangle}

\def\t{\tau}

\DeclareMathOperator{\poly}{poly}

\newenvironment{proof}{\noindent {\em {Proof:}}}{$\blacksquare$\vskip \belowdisplayskip}

\newcommand{\Ch}[1]{\hyperref[ch:#1]{Chapter~\ref*{ch:#1}}} 
\newcommand{\Sec}[1]{\hyperref[sec:#1]{Section~\ref*{sec:#1}}} 
\newcommand{\Eqn}[1]{\hyperref[eq:#1]{(\ref*{eq:#1})}} 
\newcommand{\Fig}[1]{\hyperref[fig:#1]{Figure~\ref*{fig:#1}}} 
\newcommand{\Tab}[1]{\hyperref[tab:#1]{Table~\ref*{tab:#1}}} 
\newcommand{\Thm}[1]{\hyperref[thm:#1]{Theorem~\ref*{thm:#1}}} 
\newcommand{\Lem}[1]{\hyperref[lem:#1]{Lemma~\ref*{lem:#1}}} 
\newcommand{\Prop}[1]{\hyperref[prop:#1]{Prop.~\ref*{prop:#1}}} 
\newcommand{\Cor}[1]{\hyperref[cor:#1]{Corollary~\ref*{cor:#1}}} 
\newcommand{\Def}[1]{\hyperref[def:#1]{Definition~\ref*{def:#1}}} 
\newcommand{\Alg}[1]{\hyperref[alg:#1]{Algorithm~\ref*{alg:#1}}} 
\newcommand{\Ex}[1]{\hyperref[ex:#1]{Example~\ref*{ex:#1}}} 
\newcommand{\Clm}[1]{\hyperref[clm:#1]{Claim~\ref*{clm:#1}}} 
\newcommand{\App}[1]{\hyperref[#1]{Appendix~\ref*{#1}}} 
\newcommand{\Res}[1]{\hyperref[res:#1]{Result~\ref*{res:#1}}} 
\newcommand{\Rem}[1]{\hyperref[rem:#1]{Remark~\ref*{rem:#1}}} 

\newcommand{\Secc}[1]{\hyperref[#1]{Section~\ref*{#1}}} 
\newcommand{\Seccs}[1]{\hyperref[#1]{Sections~\ref*{#1}}} 
\newcommand{\Appp}[1]{\hyperref[#1]{Appendix~\ref*{#1}}} 
\newcommand{\Lemm}[1]{\hyperref[#1]{Lemma~\ref*{#1}}} 
\newcommand{\Deff}[1]{\hyperref[#1]{Definition~\ref*{#1}}} 
\newcommand{\Thmm}[1]{\hyperref[#1]{Theorem~\ref*{#1}}} 
\newcommand{\Propp}[1]{\hyperref[#1]{Proposition~\ref*{#1}}} 
\newcommand{\Figg}[1]{\hyperref[#1]{Figure~\ref*{#1}}} 
\newcommand{\Labb}[2]{\hyperref[#2]{#1~\ref*{#2}}} 

\title{A PAC Approach to Application-Specific Algorithm Selection%
\thanks{A preliminary version of this paper appeared in the Proceedings of the
  7th Innovations in Theoretical Computer Science Conference, January 2016.  This research was supported in part
    by NSF Awards CCF-1215965 and CCF-1524062.}}
\date{}

\author{
  Rishi Gupta  \\\\
    Stanford University\\ {\tt rishig@cs.stanford.edu} \and
  Tim Roughgarden \\\\
    Stanford University\\ {\tt tim@cs.stanford.edu}}

\begin{document}

\maketitle

\begin{abstract}
The best algorithm for a computational problem
generally depends on the ``relevant inputs,'' a concept that depends
on the application domain and often defies formal
articulation.  While there is a large literature on
empirical approaches to selecting the best algorithm for a given
application domain, there has been surprisingly little theoretical
analysis of the problem.

This paper adapts concepts from statistical and online learning theory
to reason about application-specific algorithm selection.
Our models capture several state-of-the-art empirical and
theoretical approaches to the problem, ranging from self-improving
algorithms to
empirical performance models, and our results identify conditions
under which these approaches are guaranteed to perform well.
We present one framework that models algorithm selection as a
statistical learning problem, and our work here
shows that dimension notions
from statistical learning theory, historically used to measure the
complexity of classes of binary- and real-valued functions, are
relevant in a much broader algorithmic context.
We also study the online version of the algorithm selection problem,
and give possibility and impossibility results for the existence of
no-regret learning algorithms.
\end{abstract}


\section{Introduction}







Rigorously comparing algorithms is hard.
The most basic reason for this is that two different algorithms for a
computational problem generally have incomparable
performance: one algorithm
is better on some inputs, but worse on the others.
How can a theory advocate one of the algorithms over the other?
The simplest and most common solution in the theoretical analysis of
algorithms is to summarize the performance of an algorithm using a
single number, such as its worst-case performance or its average-case
performance with respect to an input distribution.
This approach effectively advocates using the algorithm with the best
summarizing value (e.g., the smallest worst-case running time).

Solving a problem ``in practice'' generally means identifying an
algorithm that works well for most or all instances of interest.
When the ``instances of interest'' are easy
to specify formally in advance --- say, planar graphs
---  the traditional analysis approaches
often give accurate performance predictions and identify useful
algorithms.
However, instances of interest commonly
possess domain-specific features that defy formal articulation.
Solving a problem in practice can require selecting an algorithm
that is optimized for the specific application domain, even though the
special structure of its instances is not well understood.
While there is a large literature, spanning numerous communities, on
empirical approaches to algorithm selection
(e.g.~\cite{fink,H+14,horvitz,huang,kotthoff,LNS09}), there
has been surprisingly little theoretical analysis of the problem.
One possible explanation is that worst-case analysis, which is the
dominant algorithm analysis paradigm in theoretical computer science,
is deliberately application-agnostic.

This paper demonstrates that application-specific algorithm selection
can be usefully modeled as a learning problem.
Our models are straightforward to understand, but also expressive
enough to capture several existing approaches in the
theoretical computer science and AI communities, ranging from the
design and analysis of self-improving algorithms~\cite{sesh} to the
application of empirical performance models~\cite{H+14}.

We present one framework that models algorithm selection as a
statistical learning problem in the spirit of Haussler~\cite{haussler}.
We prove that many useful
families of algorithms,
including broad classes of greedy and local search heuristics,
have small pseudo-dimension and hence low generalization error.
Previously,
the pseudo-dimension (and the VC dimension,
fat shattering dimension, etc.) has been used almost exclusively
to quantify the complexity of classes of prediction functions
(e.g.~\cite{haussler,AB}).\footnote{A few exceptions:
Srebro and Ben-David~\cite{SB06} use the pseudo-dimension to study the
problem of learning a good kernel for use in a support vector machine,
Long~\cite{long}
  parameterizes the performance of the randomized rounding of packing
  and covering linear programs by the pseudo-dimension of a set
  derived from the constraint matrix, and
  Mohri and Medina~\cite{MM14} and Morgenstern and
  Roughgarden~\cite{MR15} use dimension notions from learning theory
  to bound the
  sample   complexity of learning approximately revenue-maximizing
  truthful auctions.}
Our results demonstrate that this concept is
useful and relevant in a much broader algorithmic context.
It also offers a novel approach to
formalizing the oft-mentioned but rarely-defined
``simplicity''  of a family of algorithms.

We also study regret-minimization in the online version of the
algorithm selection problem.  We show that the ``non-Lipschitz''
behavior of natural algorithm classes precludes learning algorithms
that have no regret in the worst case, and prove positive results
under smoothed analysis-type assumptions.

\paragraph{Paper Organization}
\Secc{s:scenarios} outlines a number of concrete problems that
motivate the present work, ranging from greedy heuristics to SAT
solvers, and from self-improving algorithms to parameter tuning.  The
reader interested solely in the technical development can skip this
section with little loss.
\Secc{s:basic} models the task of determining the best
application-specific algorithm as a PAC learning problem, and brings
the machinery of statistical learning theory to bear on a wide class
of problems, including greedy heuristic selection, sorting,
and gradient descent step size selection.
A time-limited reader can glean the gist of
our contributions from \Seccs{ss:basic}--\ref{sss:greedy_erm}.
\Secc{s:noregret} considers the problem of learning an
application-specific algorithm online, with the goal of minimizing
regret.
\Seccs{ss:lb} and~\ref{ss:ub} present negative and positive
results for worst-case and smoothed instances, respectively.
\Secc{s:conc} concludes with a number
of open research directions.

\enlargethispage*{\baselineskip}
\section{Motivating Scenarios}\label{s:scenarios}

Our learning framework sheds light on several well-known approaches,
spanning disparate application domains, to
the problem of learning a good algorithm from data.
To motivate and provide interpretations of our results, we describe
several of these in detail.

\subsection{Example \#1: Greedy Heuristic Selection}\label{ss:greedy}


One of the most common and also most challenging motivations for
algorithm selection is presented by computationally difficult
optimization problems.  When the available computing resources are
inadequate to solve such a problem exactly, heuristic algorithms must
be used.  For most hard problems, our understanding of when
different heuristics work well remains primitive.
For concreteness, we describe one current and high-stakes example of
this issue, which
also aligns well with our model and results in
\Secc{ss:greedy2}.  The computing and operations research
literature has many similar examples.

The FCC is currently (in 2016) running a novel double auction to buy back
licenses for spectrum from certain television broadcasters and resell
them to telecommunication companies for wireless broadband
use. The auction is expected to generate over \$20 billion dollars for
the US government~\cite{fcc}.
The ``reverse'' (i.e., buyback) phase of the
auction must determine
which stations to buy out (and what to pay them).  The auction is
tasked with buying out sufficiently many stations so that the
remaining stations (who keep their licenses) can be ``repacked'' into a
small number of channels, leaving a target number of channels free
to be repurposed for wireless broadband.
To first order, the feasible repackings are determined by
interference constraints between stations.
Computing a repacking therefore resembles familiar hard
combinatorial problems like the independent set and graph coloring
problems.
The reverse auction
uses a greedy heuristic to compute the order in
which stations are removed from the reverse auction (removal means the
station keeps its license)~\cite{MS14}.
The chosen heuristic
favors stations with high value, and discriminates against stations that
interfere with a large number of other stations.\footnote{Analogously,
  greedy heuristics for the maximum-weight independent set problem
  favor vertices with higher weights and with lower
  degrees~\cite{STY03}.  Greedy heuristics for
  welfare maximization in combinatorial auctions prefer bidders with
  higher values and smaller demanded bundles~\cite{LOS02}.}
There are many ways of combining these two criteria,
and no obvious reason to favor one specific implementation over another.
The specific implementation in the FCC auction
has been justified through
trial-and-error experiments using synthetic instances that are thought
to be representative~\cite{MS14}.
One interpretation of our results in \Secc{ss:greedy2} is as a
post hoc
justification of this exhaustive approach for sufficiently simple
classes of algorithms, including the greedy heuristics considered for
this FCC auction.

\subsection{Example \#2: Self-Improving Algorithms}\label{ss:sorting}


The area of {\em self-improving algorithms} was initiated by
Ailon et al.\cite{sesh}, who considered sorting and clustering
problems.  Subsequent
work~\cite{CS08,CMS10,CMS12} studied several problems in low-dimensional
geometry, including the maxima and convex hull problems.
For a given problem,
the goal is to design an algorithm that,
given a sequence of i.i.d.\ samples from an unknown distribution over
instances, converges to the optimal algorithm for that distribution.
In addition, the algorithm should use only a small amount of
auxiliary space.
For example, for sorting independently distributed array
entries, the algorithm in~Ailon et al.\cite{sesh} solves each instance (on $n$
numbers) in $O(n \log n)$ time, uses space $O(n^{1+c})$
(where $c > 0$ is an arbitrarily small constant),
and after a polynomial number of
samples has expected running time within a constant factor of that of an
information-theoretically optimal algorithm for the unknown
input distribution.  \Secc{ss:sorting2} reinterprets
self-improving algorithms via our general framework.

\subsection{Example \#3: Parameter Tuning
in Optimization and Machine Learning}\label{ss:gd}

Many ``algorithms'' used in practice are really meta-algorithms, with
a large number of free parameters that need to be instantiated by the
user.  For instance, implementing even in the most basic version of
gradient descent requires choosing a step size and error tolerance.
For a more extreme version, CPLEX, a widely-used commercial linear and
integer programming solver, comes with a 221-page parameter reference
manual describing 135 parameters~\cite{xu11}.

An analogous problem in machine learning is ``hyperparameter
optimization,'' where the goal is to tune the parameters of a learning
algorithm so that it learns (from training data) a model with high
accuracy on test data, and in particular a model that does not overfit
the training data.  A simple example is regularized regression,
such as ridge regression, where a single parameter governs the
trade-off between the accuracy of the learned model on training data and
its ``complexity.''  More sophisticated learning algorithms can have
many more parameters.

Figuring out the ``right'' parameter values is notoriously challenging
in practice.  The CPLEX manual simply advises that ``you may need to
experiment with them.''  In machine
learning, parameters are often set by discretizing and then
applying
brute-force search (a.k.a.\ ``grid search''), perhaps with random
subsampling (``random search'')~\cite{BB12}.  When this is
computationally infeasible, variants of gradient descent are often
used to explore the parameter space, with no guarantee of converging
to a global optimum.

The results in \Secc{ss:gd2} can be interpreted as
a sample
complexity analysis of grid search for the problem of choosing the
step size in gradient descent to minimize the expected number of
iterations needed for convergence.
We view this as a first step toward reasoning more generally about the
problem of learning good parameters for machine learning algorithms.

\subsection{Example \#4: Empirical Performance Models for SAT
  Algorithms}\label{ss:sat}

The examples above already motivate selecting an algorithm for a
problem based on characteristics of the application domain.
A more ambitious and refined approach is to select an algorithm on
a {\em per-instance} (instead of a per-domain) basis.
While it's impossible to memorize the best algorithm for every possible
instance, one might hope to use coarse {\em features} of a problem
instance as a guide to which algorithm is likely to work well.

For example, Xu et al.~\cite{xu} applied this idea to the
satisfiability (SAT) problem.
Their algorithm portfolio consisted of seven
state-of-the-art SAT solvers with incomparable and
widely varying running times across different instances.
The authors identified a number of instance features,
ranging from simple
features like input size and clause/variable ratio, to complex
features like Knuth's estimate of the search tree size~\cite{knuth} and
the rate of progress of local search probes.\footnote{It is important,
  of course, that computing the features of an instance
is an easier problem than solving it.}
The next step involved building an ``empirical performance model'' (EPM) for
each of the seven algorithms in the portfolio --- a mapping from instance
feature vectors to running time predictions.  They then
computed their
EPMs using labeled training data and a suitable regression model.
With the EPMs in hand,
it is clear how to perform per-instance algorithm selection: given an
instance, compute its features,
use the EPMs to predict the running time of each algorithm
in the portfolio, and run the algorithm with the smallest predicted
running time.  Using these ideas (and several optimizations), their
``SATzilla'' algorithm won numerous medals at the 2007
SAT Competition.\footnote{See Xu et al.~\cite{xu2} for details on the latest
  generation of their solver.}
\Secc{ss:features} outlines how to extend our PAC learning
framework to reason about EPMs and feature-based algorithm selection.

\section{PAC Learning an Application-Specific Algorithm}\label{s:basic}

This section casts the problem of selecting the best
algorithm for a poorly understood application domain as one of
learning the
optimal algorithm with respect to an unknown instance distribution.
\Secc{ss:basic} formally defines the basic model,
\Secc{ss:pseudo} reviews
relevant preliminaries from statistical learning theory,
\Secc{ss:greedy2} bounds the pseudo-dimension of many classes
of greedy and local search heuristics,
\Secc{ss:sorting2} re-interprets the theory of self-improving
algorithms via our framework,
\Secc{ss:features} extends the basic model to capture empirical
performance models and feature-based algorithm selection,
and \Secc{ss:gd2} studies step size selection in gradient descent.


\subsection{The Basic Model}\label{ss:basic}

Our basic model consists of the following ingredients.
\begin{enumerate}

\item A fixed computational or optimization problem $\Probl$.  For
  example, $\Probl$ could be computing a maximum-weight independent
  set of a graph (\Secc{ss:greedy}), or sorting $n$ elements
  (\Secc{ss:sorting}).

\item An unknown distribution $\distr$ over instances $x \in
\Probl$.

\item A set $\A$ of algorithms for $\Probl$; see
  \Seccs{ss:greedy2} and~\ref{ss:sorting2} for concrete
  examples.

\item A performance measure $\cost:\A \times \Probl \rightarrow [0,H]$
  indicating the performance of a given algorithm on a given instance.
Two common choices for $\cost$ are the running time of an algorithm,
and, for optimization problems, the objective function value of the
solution produced by an algorithm.


\end{enumerate}

The ``application-specific information'' is encoded by the unknown
input distribution $\distr$, and the corresponding ``application-specific
optimal algorithm'' $A_{\distr}$ is the algorithm that minimizes or maximizes (as appropriate)
$\expect[x \in \distr]{\cost(A,x)}$ over $A \in \A$.
The {\em error} of an algorithm $A \in \A$ for a
distribution $\distr$ is
$$
\bigl\lvert\expect[x \sim \distr]{\cost(A,x)}
-
\expect[x \sim \distr]{\cost(A_{\distr},x)}\bigr\rvert.
$$

In our basic model, the goal is:
\begin{itemize}

\item [] {\em Learn the application-specific optimal algorithm from data
  (i.e., samples from $\distr$)}.

\end{itemize}
More precisely, the learning algorithm is given $m$ i.i.d.\ samples
$x_1,\ldots,x_m \in \Probl$ from $\dist$, and (perhaps implicitly) the
corresponding performance
$\cost(A,x_i)$ of each algorithm $A \in \A$ on each input
$x_i$.
The learning algorithm uses this information to suggest an algorithm
$\hat{A} \in \A$ to use on future inputs drawn from~$\dist$.
We seek learning algorithms that almost always output an algorithm of
$\A$ that performs almost as well as the optimal algorithm in $\A$ for
$\dist$.

\begin{definition}\label{d:learnopt}
A learning algorithm $L$ {\em $(\eps,\delta)$-learns the
  optimal algorithm in $\A$ from $m$ samples} if, for every distribution
$\dist$ over $\Probl$, with probability at least $1-\delta$ over $m$
samples $x_1,\ldots,x_m \sim \dist$, $L$ outputs an algorithm
$\hat{A} \in \A$ with error at most $\eps$.
\end{definition}


\subsection{Pseudo-Dimension and Uniform Convergence}\label{ss:pseudo}

PAC learning an optimal algorithm, in the sense
of \Deff{d:learnopt}, reduces to bounding the ``complexity''
of the class $\A$ of algorithms.
We next review the relevant definitions from statistical learning
theory.

Let $\cH$ denote a set of real-valued functions defined on the set $X$.
A finite subset $S = \{x_1,\ldots,x_m\}$ of $X$ is {\em
  (pseudo-)shattered} by $\cH$ if there exist real-valued {\em
  witnesses} $r_1,\ldots,r_m$ such that, for each of the $2^m$ subsets
$T$ of $S$, there exists a function $h \in \cH$ such that
$h(x_i) > r_i$ if and only if $i \in T$ (for $i=1,2,\ldots,m$).
The {\em pseudo-dimension} of $\cH$ is the cardinality of the largest
subset shattered by $\cH$ (or $+\infty$, if arbitrarily large finite
subsets are shattered by $\cH$).  The pseudo-dimension is a natural
extension of the VC dimension from binary-valued to real-valued
functions.%
\footnote{The {\em fat shattering dimension} is another common
  extension of the VC dimension to real-valued functions. It is a weaker
  condition, in that the fat shattering dimension of $\mathcal{H}$ is
  always at most the pseudo-dimension of $\mathcal{H}$, that is still
  sufficient for sample complexity bounds. Most of our arguments give
  the same upper bounds on pseudo-dimension and fat shattering
  dimension, so we present the stronger statements.}

To bound the sample complexity of accurately estimating the
expectation of all functions in $\cH$, with respect to an arbitrary
probability distribution $\dist$ on $X$, it is enough to bound the
pseudo-dimension of $\cH$.
\begin{theorem}[{{Uniform Convergence (e.g.~\cite{AB})}}]\label{t:pseudo}
Let $\cH$ be a class of functions with domain $X$ and
range in $[0,H]$, and suppose $\cH$ has pseudo-dimension $d_{\cH}$.
For every distribution $\dist$ over $X$, every $\eps > 0$, and every $\delta \in
(0,1]$, if
\begin{equation}\label{eq:m}
m \ge
c\left(\frac{H}{\epsilon}\right)^2\left(d_{\cH}
      + \ln\left(\frac{1}{\delta}\right)\right)
\end{equation}
for a suitable constant $c$ (independent of all other
parameters),
then with probability at least $1-\delta$ over $m$ samples
$x_1,\ldots,x_m \sim \dist$,
$$
\left| \left( \frac{1}{m} \sum_{i=1}^m h(x_i) \right) - \expect[x \sim
  \dist]{h(x)} \right| < \eps
$$
for every $h \in \cH$.
\end{theorem}

We can identify each algorithm $A \in \A$ with the real-valued function $x
\mapsto \cost(A,x)$. Regarding the class $\A$ of algorithms as a set of
real-valued functions defined on $\Probl$, we can discuss its pseudo-dimension,
as defined above. We need one more definition before we can apply our machinery
to learn algorithms from $\mathcal{A}$.

\begin{definition}[Empirical Risk Minimization (ERM)]\label{d:erm}
 Fix an optimization problem $\Pi$, a performance measure $\cost$, and a set of
 algorithms $\mathcal{A}$. An algorithm $L$ is an {\em ERM algorithm} if, given any
 finite subset $S$ of $\Pi$, $L$ returns an (arbitrary) algorithm from $\mathcal{A}$ with
 the best average performance on $S$.
\end{definition}

For example, for any $\Pi$, $\cost$, and finite $\mathcal{A}$, there
is the trivial
ERM algorithm that simply computes the average performance of each algorithm on
$S$ by brute force, and returns the best one.  The next corollary follows easily from
\Deff{d:learnopt}, \Thmm{t:pseudo}, and \Deff{d:erm}.

\begin{corollary}\label{cor:ub}
Fix parameters $\eps>0$, $\delta \in (0,1]$, a set of problem
  instances $\Pi$, and a performance measure $\cost$.  Let $\A$ be a
  set of algorithms that has pseudo-dimension $d$ with respect to
  $\Pi$ and $\cost$. Then any ERM algorithm $(2\eps,\delta)$-learns
  the optimal algorithm in $\A$ from $m$ samples, where $m$ is defined
  as in~\eqref{eq:m}.
\end{corollary}

\Cor{ub} is only interesting if interesting classes of
algorithms $\mathcal{A}$ have small pseudo-dimension.
 In the simple case where $\A$ is finite, as in our example of an algorithm
portfolio for SAT (\Secc{ss:sat}), the
pseudo-dimension of $\A$ is trivially at most $\log_2 |\A|$.  The
following
sections demonstrate the much less obvious fact that natural infinite classes of
algorithms also have small pseudo-dimension.

\begin{remark}[Computational Efficiency]
The present work focuses on the sample complexity rather than the
  computational aspects of learning, so outside of a few remarks we won't say
  much about the existence or efficiency of ERM in our examples. A priori, an
  infinite class of algorithms may not admit any ERM algorithm at all, though all
  of the examples in this paper do have ERM algorithms under mild
  assumptions.
\end{remark}

\subsection{Application: Greedy Heuristics and Extensions}\label{ss:greedy2}

The goal of this section is to bound the pseudo-dimension of many
classes of greedy
heuristics including, as a special case, the family of heuristics relevant
for the FCC double auction described in \Secc{ss:greedy}.  It
will be evident that analogous computations are possible for many
other classes of heuristics, and we provide several extensions in
\Secc{sss:greedy} to illustrate this point.
Throughout this section, the performance measure $\cost$ is the
objective function value of the solution produced by a heuristic on an
instance, where we assume without loss of generality a maximization
objective.

\subsubsection{Definitions and Examples}\label{sss:def}

Our general definitions are motivated by greedy heuristics for
($NP$-hard) problems like the following; the reader will have no
difficulty coming up with additional natural examples.

\begin{enumerate}

\item {\em Knapsack.}  The input is $n$ items
  with values
  $v_1,\ldots,v_n$, sizes $s_1,\ldots,s_n$, and a knapsack capacity
  $C$.  The goal is to compute a subset $S \sse \{1,2,\ldots,n\}$ with
  maximum total value $\sum_{i \in S} v_i$, subject to having total
  size $\sum_{i \in S} s_i$ at most $C$.
Two natural greedy heuristics are to greedily pack items (subject to
feasibility) in order of nonincreasing value $v_i$, or in order of
nonincreasing density $v_i/s_i$ (or to take the better of the two, see
\Secc{sss:greedy}).


\item {\em Maximum-Weight Independent Set (MWIS).}
The input is an undirected graph $G=(V,E)$ and
a non-negative weight
$w_v$ for each vertex $v \in V$.  The goal is to compute the independent set ---
a subset of mutually non-adjacent vertices --- with maximum total
weight.  Two natural greedy heuristics are to greedily choose vertices
(subject to feasibility) in order of nonincreasing weight $w_v$, or
nonincreasing density $w_v/(1+\deg(v))$.  (The intuition for the
denominator is that choosing $v$ ``uses up'' $1+\deg(v)$ vertices ---
$v$ and all of its neighbors.)  The latter heuristic also has a
(superior) adaptive variant, where the degree $\deg(v)$
is computed in
the subgraph induced by the vertices not yet blocked from
consideration, rather than in the original graph.\footnote{An
  equivalent description is: whenever a vertex $v$ is added to the
  independent set, delete $v$ and its neighbors from the graph, and
  recurse on the remaining graph.}

\item {\em Machine Scheduling.}
This is a family of optimization
  problems, where $n$ jobs with various attributes (processing
  time, weight, deadline, etc.) need to be
  assigned to $m$ machines, perhaps subject to some constraints
  (precedence constraints, deadlines, etc.), to optimize some
  objective (makespan, weighted sum of completion times, number of
  late jobs, etc.).  A typical greedy heuristic for such a problem
  considers jobs in some order according to a score derived from the
  job parameters (e.g., weight divided by processing time), subject to
  feasibility, and always
  assigns the current job to the machine that currently has
  the lightest load (again, subject to feasibility).

\end{enumerate}

In general, we consider {\em object assignment problems}, where the
input is a set of $n$ objects with various attributes,
and the feasible solutions consist of assignments of the objects to a
finite set $R$, subject to feasibility constraints.
The attributes of an object are represented as an element $\xi$
of an abstract set.
For example, in the Knapsack problem $\xi$ encodes the
value and size of an object; in the MWIS problem, $\xi$ encodes
the weight and (original or residual) degree of a vertex.
In the Knapsack and MWIS problems, $R = \{0,1\}$, indicating whether
or not a given object is selected.
In machine scheduling problems, $R$
could be $\{1,2,\ldots,m\}$, indicating the machine to which a job is
assigned, or a richer set that also keeps track of the job ordering on
each machine.

By a {\em greedy heuristic}, we mean algorithms of the following form (cf., the
``priority algorithms'' of Borodin et al.~\cite{borodin}):
\begin{enumerate}

\item While there remain unassigned objects:

\begin{enumerate}

\item Use a {\em scoring rule} $\sigma$ (see below) to compute a score
  $\sigma(\xi_i)$ for each unassigned object $i$, as a function of its
  current attributes $\xi_i$.

\item For the unassigned object $i$ with the highest score, use an
  {\em assignment rule} to assign $i$ a value from $R$ and,
  if necessary, update the attributes of the other unassigned
  objects.\footnote{We
    assume that there is always as least one choice of assignment that
    respects the feasibility constraints; this holds for all of our
    motivating examples.}
For concreteness, assume that ties are always resolved
lexicographically.

\end{enumerate}

\end{enumerate}

A {\em scoring rule} assigns a real number to an object
as a function of its attributes.
Assignment rules that do not modify objects' attributes yield
non-adaptive greedy heuristics, which use only the original
attributes of each object
(like $v_i$ or $v_i/s_i$ in the
Knapsack problem, for instance).  In this case,
objects' scores
can be computed in advance of the main loop of the greedy heuristic.
Assignment rules that modify object attributes yield adaptive
greedy heuristics, such as the adaptive MWIS heuristic described above.

In a {\em single-parameter} family of scoring rules,
there is a scoring rule of the form $\sigma(\rho,\xi)$ for each
parameter value $\rho$ in some interval $I \sse \RR$.
Moreover, $\sigma$ is assumed to be continuous in $\rho$ for
each fixed value of $\xi$.
Natural examples include Knapsack scoring
rules of the form $v_i/s_i^{\rho}$ and MWIS scoring rules of the form
$w_v/(1+\deg(v))^{\rho}$ for $\rho \in [0,1]$ or $\rho \in
[0,\infty)$.
A single-parameter family of scoring rules is {\em
  $\kappa$-crossing} if, for each distinct pair of attributes $\xi,\xi'$,
there are at most
$\kappa$ values of $\rho$ for which $\sigma(\rho,\xi) =
\sigma(\rho,\xi')$.  For example, all of the scoring rules mentioned
above are 1-crossing rules.

For an example assignment rule, in the Knapsack and MWIS problems, the rule
simply assigns $i$ to ``1'' if it is feasible to do so, and to ``0''
otherwise.  A typical machine scheduling assignment rule assigns the
current job to the machine with the lightest load.
In the adaptive greedy heuristic for the MWIS problem, whenever the
assignment rule assigns ``1'' to a vertex $v$, it updates the residual
degrees of other unassigned vertices (two hops away) accordingly.

We call an assignment rule {\em $\beta$-bounded} if every object
$i$ is guaranteed to take on at most $\beta$ distinct attribute values.
For example, an assignment rule that never
modifies an object's attributes is 1-bounded.  The assignment rule in
the adaptive
MWIS algorithm is $n$-bounded, since it only modifies the degree of a
vertex (which lies in $\{0,1,2\ldots,n-1\}$).

Coupling a single-parameter family of $\kappa$-crossing scoring rules
with a fixed $\beta$-bounded assignment rule
yields a {\em $\kb$-single-parameter family of greedy
  heuristics}.  All of our running examples of greedy heuristics are
$(1,1)$-single-parameter families, except for the adaptive MWIS
heuristic, which is a $(1,n)$-single-parameter family.

\subsubsection{Upper Bound on Pseudo-Dimension}

We next show that every $\kb$-single-parameter
family of greedy heuristics has small pseudo-dimension.
This result
applies to all of the concrete examples mentioned above,
and it is easy to come up with other examples (for the problems already
discussed, and for additional problems).

\begin{theorem}[Pseudo-Dimension of Greedy Algorithms]\label{t:greedy}
If $\A$ denotes a $\kb$-single-para\-meter family of greedy
heuristics for an object assignment problem with $n$ objects, then the
pseudo-dimension of $\A$ is $O(\log (\kappa\beta n))$.
\end{theorem}

In particular, all of our running examples are classes
of heuristics with pseudo-dimension $O(\log n)$.

\vspace{.1in}

\begin{proof}
Recall from the definitions (\Secc{ss:pseudo}) that we need
to upper bound the size of every set that is shatterable using
the greedy heuristics in $\A$.
For us, a set is a fixed set of $s$ inputs (each with $n$ objects)
$S = x_1,\ldots,x_s$.
For a potential witness $r_1,\ldots,r_s \in \RR$,
every algorithm $A \in \A$ induces a binary labeling of each
sample $x_i$, according to whether $\cost(A,x_i)$ is strictly more
than or at most $r_i$.
We proceed to bound from above the number of distinct binary
labellings of $S$ induced by the algorithms of $\A$, for any
potential witness.

Consider ranging over algorithms $A \in \A$ --- equivalently, over
parameter values $\rho \in I$.
The trajectory of a greedy heuristic $A \in \A$ is uniquely
determined by the outcome of the comparisons between the current
scores of the unassigned objects in each iteration of the algorithm.
Since the family uses a $\kappa$-crossing scoring rule, for every pair
$i,j$ of distinct objects and possible attributes $\xi_i,\xi_j$,
there are at most $\kappa$ values of $\rho$ for which there is a tie
between the score of $i$ (with attributes $\xi_i$) and that of $j$
(with attributes $\xi_j$).  Since $\sigma$ is continuous in
$\rho$, the relative order of the score of $i$ (with $\xi_i$) and $j$
(with $\xi_j$) remains the same in the open interval between
two successive values of $\rho$ at which their scores are
tied.
The upshot is that we can partition $I$ into at most
$\kappa+1$ intervals such that the outcome of the comparison between $i$
(with attributes $\xi_i$) and $j$ (with attributes $\xi_j$) is
constant on each interval.\footnote{This argument assumes that $\xi_i
  \neq \xi_j$.  If $\xi_i =
  \xi_j$, then because we break ties between equal scores
  lexicographically, the outcome of the comparison between
  $\sigma(\xi_i)$ and $\sigma(\xi_j)$ is in fact constant on the entire
  interval $I$ of parameter values.}

Next, the $s$ instances of $S$ contain a total of $sn$ objects.
Each of these objects has some initial attributes.
Because the assignment rule is $\beta$-bounded, there are at most
$sn\beta$ object-attribute pairs $(i,\xi_i)$ that could possibly
arise in the execution of any algorithm from $\A$ on any instance of
$S$.  This implies that, ranging across all algorithms of $\A$ on all
inputs in $S$, comparisons are only ever made between at most
$(sn\beta)^2$ pairs of object-attribute pairs (i.e., between an object
$i$ with current attributes $\xi_i$ and an object $j$ with current
attributes $\xi_j$).  We call these the {\em relevant comparisons}.

For each relevant comparison, we can partition $I$ into at most $\kappa+1$
subintervals such that the comparison outcome is constant (in $\rho$)
in each subinterval.  Intersecting the partitions of all of the
at most $(sn\beta)^2$ relevant comparisons splits $I$ into at most
$(sn\beta)^2\kappa + 1$ subintervals such that {\em every}
relevant comparison
is constant in each subinterval.  That is, all of the algorithms of
$\A$ that correspond to the parameter values $\rho$ in such a
subinterval execute identically on every input in $S$.  The number of
binary labellings of $S$ induced by algorithms of $\A$ is trivially at
most the number of such subintervals.
Our upper bound $(sn\beta)^2\kappa + 1$ on the number of subintervals
exceeds $2^s$, the requisite number of labellings to shatter $S$,
only if $s = O(\log (\kappa\beta n))$.
\end{proof}

\Thmm{t:greedy} and \Cor{ub} imply that, if
$\kappa$ and $\beta$ are bounded above by a polynomial in $n$,
then an ERM algorithm $\ed$-learns the optimal algorithm in~$\A$ from only
$m = \tilde{O}(\tfrac{H^2}{\eps^2})$ samples,%
\footnote{The notation $\tilde{O}(\cdot)$ suppresses logarithmic factors.}
where $H$ is the largest
objective function value of a feasible solution output by an algorithm
of $\A$ on an instance of $\Probl$.\footnote{Alternatively, the
  dependence of $m$ on $H$ can be removed if learning error $\eps H$
  (rather than $\eps$) can be tolerated --- for example, if the
  optimal objective function value is expected to be proportional to
  $H$ anyways.}

We note that \Thmm{t:greedy} gives a quantifiable sense in
which natural greedy algorithms are indeed ``simple
algorithms.''  Not all classes of algorithms have such a small
pseudo-dimension; see also the next section for further
discussion.\footnote{When the performance measure $\cost$ is solution
  quality, as in this section, one cannot identify
``simplicity'' with ``low pseudo-dimension'' without caveats:
strictly speaking, the set $\A$ containing only
  the optimal algorithm for the problem has pseudo-dimension~1.
When the problem $\Pi$ is $NP$-hard and $\A$ consists only of
polynomial-time algorithms (and assuming $P \neq NP$), the
pseudo-dimension
is a potentially relevant complexity measure for the
heuristics in $\A$.}

\begin{remark}[Non-Lipschitzness]\label{rem:lip}
We noted in \Secc{ss:pseudo} that the pseudo-dimension of a
finite set $\A$ is always at most $\log_2 |\A|$.  This suggests a
simple discretization approach to learning the best algorithm
from~$\A$: take a finite ``$\eps$-net'' of $\A$ and learn the best
algorithm in the finite net.  (Indeed, \Secc{ss:gd2} uses
precisely this approach.)
The issue is that without some kind of Lipschitz condition --- stating
that ``nearby'' algorithms in $\A$ have approximately the same
performance on all instances --- there's no reason to believe that the
best algorithm in the net is almost as good as the best algorithm from
all of $\A$.  Two different greedy heuristics --- two MWIS greedy
algorithms with arbitrarily close $\rho$-values, say --- can have
completely different executions on an instance.  This lack of
a Lipschitz property explains why we take care in
\Thmm{t:greedy} to bound the pseudo-dimension of the full infinite set
of greedy heuristics.\footnote{The $\eps$-net approach has the
  potential to work for greedy algorithms that choose the next object
  using a softmax-type rule, rather than deterministically as the
  unassigned object with the highest score.}
\end{remark}

\subsubsection{Computational Considerations}\label{sss:greedy_erm}

The proof of \Thmm{t:greedy} also demonstrates the presence of an efficient ERM algorithm:
the $O((sn\beta)^2)$ relevant comparisons are easy to
identify, the corresponding subintervals induced by each are easy to
compute (under mild assumptions on the scoring rule), and brute-force
search can be used to pick the best of the resulting
$O((sn\beta)^2\kappa)$ algorithms (an arbitrary one from each
subinterval).  This algorithm runs in polynomial time as long as
$\beta$ and $\kappa$ are polynomial in~$n$, and every algorithm of
$\A$ runs in polynomial time.

For example, for the family of Knapsack scoring rules described above,
implementing this ERM algorithm reduces to comparing the outputs of
$O(n^2m)$ different greedy heuristics (on each of the $m$ sampled
inputs), with $m=O(\log n)$.
For the adaptive MWIS heuristics, where $\beta =n$, it is
enough to compare the sample performance of $O(n^4m)$ different greedy
algorithms, with $m = O(\log n)$.

\subsubsection{Extensions: Multiple Algorithms, Multiple Parameters,
  and Local Search}\label{sss:greedy}

\Thmm{t:greedy} is robust and its proof is easily modified to
accommodate various extensions.  For a first example, consider
algorithms than run $q$ different members of a single-parameter greedy
heuristic family and return the best of the $q$ feasible solutions
obtained.\footnote{For example, the classical
  $\tfrac{1}{2}$-approximation for Knapsack has this form (with $q =
  2$).}
Extending the proof of \Thmm{t:greedy} yields a
pseudo-dimension bound of $O(q \log (\kappa\beta n))$ for the class of
all such algorithms.

For a second example,
consider families of greedy heuristics parameterized by $d$ real-valued
parameters $\rho_1,\ldots,\rho_d$.
Here, an analog of \Thmm{t:greedy}
holds with the crossing number $\kappa$ replaced by a more complicated
parameter --- essentially, the number of connected components of the
co-zero set of the difference of two scoring functions (with $\xi,\xi'$
fixed and variables $\rho_1,\ldots,\rho_d$).  This number can often be
bounded (by a function exponential in $d$) in natural cases, for
example using B\'ezout's theorem (see e.g.~\cite{ag}).

For a final extension, we sketch how to adapt the definitions and
results of this section from greedy to local search heuristics.
The input is again an object assignment problem (see
\Secc{sss:def}), along with an initial feasible solution (i.e.,
an assignment of objects to $R$, subject to feasibility constraints).
By a {\em $k$-swap local search heuristic}, we mean algorithms of the
following form:
\begin{enumerate}

\item Start with arbitrary feasible solution.

\item While the current solution is not locally optimal:

\begin{enumerate}

\item Use a {\em scoring rule} $\sigma$ to compute a score
  $\sigma(\{\xi_i : i \in K\})$ for each set of objects $K$ of size
  $k$, where $\xi_i$ is the current attribute of object
  $i$.

\item For the set $K$ with the highest score, use an
  {\em assignment rule} to re-assign each $i \in K$ to a value from $R$.
  If necessary, update the attributes of the appropriate objects.
(Again, assume that ties are resolved lexicographically.)

\end{enumerate}

\end{enumerate}
We assume that the assignment rule maintains feasibility,
so that we have a feasible assignment at the end of each execution of
the loop.  We also assume that the scoring and assignment rules
ensure that the algorithm terminates, e.g.\ via the existence of
a global objective function that decreases at every iteration (or by
incorporating timeouts).

A canonical example of a $k$-swap local search heuristic is the
$k$-OPT heuristic for the traveling salesman problem (TSP)%
\footnote{Given a complete undirected
  graph with a cost $c_{uv}$ for each edge $(u,v)$,
compute a tour (visiting each vertex  exactly once) that minimizes the
sum of the edge costs.} (see e.g.~\cite{JM97}).
We can view TSP as an object assignment problem, where the objects are
edges and $R = \{0,1\}$; the feasibility constraint is that the edges
assigned to~1 should form a tour.
Recall that a local move in $k$-OPT consists of swapping
out $k$ edges from the current tour and swapping in $k$ edges to
obtain a new tour.  (So in our terminology, $k$-OPT is a $2k$-swap
local search heuristic.)
Another well-known
example is the local search algorithms for the $p$-median
problem studied in Arya et al.~\cite{A+01}, which are parameterized by
the number
of medians that can be removed and added in each local move.
Analogous local search algorithms make sense for the MWIS problem as
well.

Scoring and assignment rules are now defined on subsets of $k$
objects, rather than individual objects.
A single-parameter family of scoring rules is now called
{\em $\kappa$-crossing} if, for every subset $K$ of at most $k$
objects and each distinct pair of attribute sets
$\xi_K$ and $\xi'_K$, there are at most $\kappa$ values of $\rho$ for
which $\sigma(\rho,\xi_K) = \sigma(\rho,\xi'_K)$.
An assignment rule is now {\em $\beta$-bounded} if for every
subset $K$ of at most $k$ objects,
ranging over all possible trajectories of the local search
heuristic,
the attribute set of $K$ takes on at most $\beta$ distinct
values.
For example, in MWIS, suppose we allow two vertices $u,v$ to be
removed and two vertices $y,z$ to be added in a
single local move, and we use
the single-parameter scoring rule family
$$
\sigma_{\rho}(u,v,y,z) = \frac{w_u}{(1+\deg(u))^{\rho}} +
\frac{w_v}{(1+\deg(v))^{\rho}} - \frac{w_y}{(1+\deg(y))^{\rho}}
- \frac{w_z}{(1+\deg(z))^{\rho}}.
$$
Here $\deg(v)$ could refer to the degree of vertex $v$ in original
graph, to the number of neighbors of~$v$ that do not have any
neighbors other than~$v$ in the current independent set, etc.
In any case, since a generalized Dirichlet polynomial with $t$ terms
has at most $t-1$ zeroes (see e.g.~\cite[Corollary 3.2]{J06}), this is
a 3-crossing family.
The natural assignment rule is $n^4$-bounded.\footnote{In general,
arbitrary local search algorithms can be made $\beta$-bounded through
time-outs: if such an algorithm always halts within
  $T$ iterations, then the corresponding assignment rule is $T$-bounded.}

By replacing the number~$n$ of objects by the number~$O(n^k)$ of
subsets of at most $k$ objects in the proof of \Thmm{t:greedy},
we obtain the following.

\begin{theorem}[Pseudo-Dimension of Local Search Algorithms]\label{t:ls}
If $\A$ denotes a $\kb$-single-parameter family of $k$-swap
local search heuristics for an object assignment problem with $n$
objects, then the
pseudo-dimension of $\A$ is $O(k \log (\kappa\beta n))$.
\end{theorem}


\subsection{Application: Self-Improving Algorithms Revisited}\label{ss:sorting2}

We next give a new interpretation of the self-improving sorting
algorithm of Ailon et al.\cite{sesh}.
Namely, we show that the main result in \cite{sesh} effectively
identifies a set of sorting algorithms that simultaneously has low
representation error (for independently distributed array elements)
and small pseudo-dimension (and hence low generalization error).
Other constructions of self-improving
algorithms~\cite{sesh,CS08,CMS10,CMS12} can be likewise reinterpreted.
In contrast to \Secc{ss:greedy2}, here our performance measure
$\cost$ is related to the running time of an algorithm $A$ on an input $x$, which
we want to minimize, rather than the objective function value of the output, which we wanted to maximize.

Consider the problem of sorting $n$ real numbers in the comparison model.
By a {\em bucket-based sorting algorithm}, we mean an algorithm $A$ for
which there are ``bucket boundaries'' $b_1 < b_2 < \cdots < b_{\ell}$
such that $A$ first distributes the $n$ input elements into their
rightful buckets, and then sorts each bucket separately, concatenating
the results.
The degrees of freedom when defining such an algorithm
are: (i) the choice of the bucket boundaries; (ii) the method used to
distribute input elements to the buckets; and (iii) the method used
to sort each bucket. The performance measure $\cost$ is the number of
comparisons used by the algorithm.\footnote{Devroye~\cite{devroye}
  studies similar families of sorting algorithms, with the goal of
  characterizing the expected running time as a function of the input
  distribution.}

The key steps in the analysis in \cite{sesh} can be reinterpreted as
proving that this set of bucket-based sorting algorithms has low
representation error, in the following sense.
\begin{theorem}[{{\cite[Theorem 2.1]{sesh}}}]\label{t:sesh}
Suppose that each array element $a_i$ is drawn independently from a
distribution $\dist_i$.  Then there exists a bucket-based sorting
algorithm with expected running time at most a constant factor times
that of the optimal sorting algorithm for $\dist_1 \times \cdots
\times \dist_n$.
\end{theorem}
The proof in~\cite{sesh} establishes \Thmm{t:sesh} even
when the number $\ell$ of buckets is only $n$,
each bucket is sorted using InsertionSort, and each element
$a_i$ is distributed independently to its rightful bucket using a
search tree stored in $O(n^{c})$ bits, where $c > 0$ is an arbitrary
constant (and the running time depends on
$\tfrac{1}{c}$).\footnote{For small $c$, each search tree $T_i$ is so
  small that some searches will go unresolved; such unsuccessful
  searches are handled by a standard binary search over the buckets.}
Let $\A_{c}$ denote the set of all such bucket-based sorting algorithms.

\Thmm{t:sesh} reduces the task of learning a near-optimal
sorting algorithm to the problem of $\ed$-learning the optimal
algorithm from $\A_c$.
\Cor{ub} reduces this learning
problem to bounding the pseudo-dimension of $\A_c$.
We next prove such a bound, which effectively says that bucket-based
sorting algorithms are ``relatively simple''
algorithms.\footnote{Not all sorting algorithms are simple in the
  sense of having polynomial pseudo-dimension.
For example, the space lower bound in \cite[Lemma
    2.1]{sesh} can be adapted to show that no class of sorting
algorithms with polynomial pseudo-dimension (or fat shattering
dimension) has low representation error in the sense of
\Thmm{t:sesh} for general distributions over sorting instances,
where the array entries need not be independent.}

\begin{theorem}[Pseudo-Dimension of Bucket-Based Sorting
    Algorithms]\label{t:bucket}
The pseudo-dim\-ension of $\A_c$ is $O(n^{1+c})$.
\end{theorem}

\begin{proof}
Recall from the definitions (\Secc{ss:pseudo}) that we need
to upper bound the size of every set that is shatterable using
the bucket-based sorting algorithms in $\A_c$.
For us, a set is a fixed set of $s$ inputs (i.e., arrays of length $n$),
$S = x_1,\ldots,x_s$.
For a potential witness $r_1,\ldots,r_s \in \RR$,
every algorithm $A \in \A_c$ induces a binary labeling of each
sample $x_i$, according to whether $\cost(A,x_i)$ is strictly more
than or at most $r_i$.
We proceed to bound from above the number of distinct binary
labellings of $S$ induced by the algorithms of $\A_c$, for any
potential witness.


By definition, an algorithm from $\A_c$ is fully specified by: (i) a
choice of $n$ bucket boundaries $b_1 < \cdots < b_n$; and (ii) for
each $i=1,2,\ldots,n$, a choice of a search tree $T_i$ of size at most
$O(n^c)$ for placing $x_i$ in the correct bucket.
Call two algorithms $A,A' \in \A_c$ {\em equivalent} if their sets of
bucket boundaries $b_1,\ldots,b_n$ and $b'_1,\ldots,b'_n$ induce the
same partition of the $sn$ array elements of the inputs in $S$ ---
that is, if $x_{ij} < b_k$ if and only if $x_{ij} < b'_k$ (for all
$i,j,k$).  The number of equivalence classes of this equivalence
relation is at most $\binom{sn + n}{n} \le (sn+n)^n$.
Within an equivalence class, two algorithms that use structurally
identical search trees will have identical performance on all $s$ of
the samples.  Since the search trees of every algorithm of $\A_c$ are
described by at most $O(n^{1+c})$ bits, ranging over the algorithms of a
single equivalence class generates at most $2^{O(n^{1+c})}$ distinct
binary labellings of the $s$ sample inputs.  Ranging over all
algorithms thus generates at most $(sn+n)^n2^{O(n^{1+c})}$ labellings.
This exceeds $2^s$, the requisite number of labellings to shatter $S$,
only if $s = O(n^{1+c})$.
\end{proof}

\Thmm{t:bucket} and \Cor{ub} imply that
$m = \tilde{O}(\tfrac{H^2}{\eps^2}n^{1+c})$ samples are enough to
$\ed$-learn the optimal algorithm in~$\A_c$, where $H$ can be taken as
the ratio between the maximum and minimum running time of any
algorithm in $\A_c$ on any instance.\footnote{We again use
  $\tilde{O}(\cdot)$ to
  suppress logarithmic factors.}  Since the minimum running time is
$\Omega(n)$ and we can assume that the maximum running time is $O(n
\log n)$ --- if an algorithm exceeds this bound, we can abort it and
safely run MergeSort instead --- we obtain a sample complexity
bound of $\tilde{O}(n^{1+c})$.\footnote{In the notation of
  \Thmm{t:pseudo}, we are taking $H = \Theta(n \log n)$, $\eps =
  \Theta(n)$, and using the fact that all quantities are $\Omega(n)$
  to conclude that all running times are correctly
estimated up to a constant
  factor.  The results implicit in~\cite{sesh} are likewise for
  relative error.}

\begin{remark}[Comparison to~\cite{sesh}]
The sample complexity bound implicit in~\cite{sesh} for learning a
near-optimal sorting algorithm is $\tilde{O}(n^c)$, a linear factor
better than the $\tilde{O}(n^{1+c})$ bound implied by
\Thmm{t:bucket}.  There is good reason for this: the
pseudo-dimension bound of \Thmm{t:bucket} implies that an even
harder problem has sample complexity $\tilde{O}(n^{1+c})$, namely that
of learning a near-optimal bucket-based sorting algorithm with respect
to an {\em arbitrary} distribution over inputs, {\em even with
correlated array elements}.\footnote{When
  array elements are not independent, however, \Thmm{t:sesh}
  fails and the best bucket-based sorting algorithm might be
  more than a constant-factor worse than the optimal sorting algorithm.}
The bound of $\tilde{O}(n^c)$
in~\cite{sesh} applies only to the problem of learning a near-optimal
bucket-based sorting algorithm for an unknown input distribution with
independent array entries --- the
savings comes from the fact that all $n$
near-optimal search trees $T_1,\ldots,T_n$ can be learned
in parallel.
\end{remark}


\subsection{Application: Feature-Based Algorithm Selection}\label{ss:features}







Previous sections studied the problem of selecting a single algorithm
for use in an application domain ---  of using training data to
make an informed commitment to a single algorithm from a class
$\A$, which is then used on all future instances.
A more refined and ambitious approach is to select an algorithm
based both on previous experience {\em and on the current instance
  to be solved}.
This approach assumes,
as in the scenario in \Secc{ss:sat},
that it is feasible to quickly compute some
features of an instance and then to select an algorithm as a function
of these features.

Throughout this section, we augment the basic model of
\Secc{ss:basic} with:
\begin{enumerate}

\item [5.] A set $\F$ of possible instance feature values, and a map $f:X
  \rightarrow \F$ that computes the features of a given
  instance.\footnote{Defining a good feature set is a notoriously
    challenging and important problem, but it is beyond the scope of our
    model --- we take the set $\F$ and map~$f$ as given.}

\end{enumerate}
For instance, if $X$ is the set of SAT instances, then $f(x)$ might
encode the clause/variable ratio of the instance~$x$, Knuth's estimate of
the search tree size~\cite{knuth}, and so on.

When the set $\F$ of possible instance feature values is finite, the
guarantees for the basic model extend easily with a
linear (in $|\F|$) degradation in the pseudo-dimension.\footnote{For
  example,~\cite{xu} first predicts whether or not a given SAT
  instance
  is satisfiable or not, and then uses a ``conditional'' empirical
  performance model to choose a SAT solver.  This can be viewed as an
  example with $|\F| = 2$, corresponding to the feature values ``looks
  satisfiable'' and
  ``looks unsatisfiable.''}
To explain, we add an additional ingredient to the model.
\begin{enumerate}

\item [6.] A set $\G$ of {\em algorithm selection maps}, with each $g
  \in \G$ a function from $\F$ to $\A$.

\end{enumerate}
An algorithm selection map recommends an algorithm as a function of
the features of an instance.

We can view an algorithm selection map $g$ as a real-valued function
defined on the instance space $X$, with $g(x)$ defined as
$\cost(g(f(x)),x)$.  That is, $g(x)$ is the running time on $x$ of the
algorithm $g(f(x))$ advocated by $g$, given that $x$ has features
$f(x)$.
The basic model studied earlier is the
special case where $\G$ is the set of constant functions,
which are in correspondence with the algorithms of $\A$.

\Cor{ub} reduces bounding the sample complexity of
$\ed$-learning the best algorithm selection map of $\G$ to bounding
the pseudo-dimension of the set of real-valued functions induced by
$\G$.  When $\G$ is finite, there is a trivial upper bound of $\log_2
|\G|$.
The pseudo-dimension is also small whenever $\F$ is small and the set $\A$
of algorithms has small pseudo-dimension.\footnote{When $\G$ is the
  set of all maps from $\F$ to $\A$ and every feature value of~$\F$
  appears with approximately the same probability, one can
  alternatively
  just separately learn the best algorithm for each feature value.}

\begin{proposition}[Pseudo-Dimension of Algorithm Selection Maps]\label{t:smallb}
If $\G$ is a set of algorithm selection maps from a finite set $\F$ to
a set $\A$ of algorithms with pseudo-dimension $d$, then $\G$ has
pseudo-dimension at most $|\F|d$.
\end{proposition}

\begin{proof}
A set of inputs of size $|\F|d+1$ is shattered only if there is a
shattered set of inputs with identical features of size $d+1$.
\end{proof}




Now suppose $\F$ is very large (or infinite).
We focus on the case where
$\A$ is small enough that it is feasible to learn a
separate performance prediction model for each algorithm $A \in
\A$ (though see \Rem{largeA}).
This is exactly the approach taken
in the motivating example of empirical performance models (EPMs) for
SAT described in \Secc{ss:sat}.
In this case, we augment the basic model to include a family of
performance predictors.
\begin{enumerate}

\item [6.] A set $\cP$ of {\em performance predictors}, with each $p
  \in \cP$ a function from $\F$ to $\RR$.

\end{enumerate}
Performance predictors play the same role as the EPMs used in~\cite{xu}.

The goal is to learn, for each algorithm $A \in \A$, among
all permitted predictors $p \in \cP$, the one that minimizes some loss
function.  Like the performance measure $\cost$, we take this loss
function as given.  The most commonly used loss function is squared
error; in this case, for each $A \in \A$ we aim to compute
the function that minimizes
$$
\expect[x \sim \dist]{(\cost(A,x)-p(f(x)))^2}
$$
over $p \in \cP$.\footnote{Note that the expected loss incurred by
  the best   predictor depends   on the choices of the predictor set $\cP$,
  the feature set
  $\F$, and map $f$.  Again,  these choices are outside our model.}
For a fixed algorithm $A$,
this is a standard regression problem, with domain $\F$,
real-valued labels, and a distribution on $\F \times \RR$ induced by
$\dist$ via $x \mapsto (f(x),\cost(A,x))$.  Bounding the sample
complexity of this learning problem
reduces
to bounding the pseudo-dimension of $\cP$.
For standard choices of $\cP$, such bounds are well known.
%
For example, suppose the set $\cP$ is the class of {\em linear predictors},
with each $p \in \cP$ having the form $p(f(x))
= a^Tf(x)$ for some coefficient vector $a \in \RR^d$.\footnote{A
  linear model might sound unreasonably simple for the
  task of predicting the running time of an algorithm, but
  significant complexity can be included in
the feature map $f(x)$.
  For example, each coordinate of $f(x)$ could be a nonlinear
  combination of several ``basic features'' of $x$.  Indeed, linear models
  often exhibit surprisingly good empirical performance, given a
  judicious choice of a feature set~\cite{LNS09}.}
%
\begin{proposition}[Pseudo-Dimension of Linear Predictors]\label{prop:linear}
If $\F$ contains real-valued $d$-dim\-ensional features and $\cP$ is the
set of linear predictors, then the pseudo-dimension of $\cP$ is at
most~$d$.
\end{proposition}
If all functions in~$\cP$ map all possible $\varphi$ to $[0,H]$, then
\Propp{prop:linear} and \Cor{ub} imply a sample complexity bound
of~$\tilde{O}(\tfrac{H^4}{\eps^2}d)$ for $\ed$-learning the predictor
with minimum expected square error. 
Similar results hold, with worse dependence on~$d$, if $\cP$ is a set
of low-degree polynomials~\cite{AB}.

For another example, suppose $\cP_{\ell}$ is the set of
regression trees with at most $\ell$ nodes, where each internal node
performs an inequality test on a coordinate of the feature vector
$\varphi$ (and leaves are labeled with performance
estimates).\footnote{Regression trees, and random forests thereof,
  have emerged as a popular class of predictors in empirical work on
  application-specific algorithm selection~\cite{H+14}.}
This class also has low pseudo-dimension,
and hence the problem of learning a near-optimal
predictor has correspondingly small sample complexity. 

\begin{proposition}[Pseudo-Dimension of Regression Trees]
Suppose $\F$ contains real-valued $d$-dimensional features and let
$\cP_{\ell}$ be the set of regression trees with at most $\ell$
nodes, where each node performs an inequality test on one of the features.
Then, the pseudo-dimension of $\cP_{\ell}$ is
$O(\ell \log (\ell d))$.
\end{proposition}





\begin{remark}[Extension to Large $\A$]\label{rem:largeA}
We can also extend our approach to scenarios with a large or infinite
set~$\A$ of possible algorithms.
This extension is relevant to state-of-the-art empirical approaches to
the auto-tuning of algorithms with many parameters, such as
mathematical programming solvers~\cite{H+14}; see also
the discussion in \Secc{ss:gd}.  (Instantiating all of
the parameters yields a fixed algorithm; ranging over all possible
parameter values yields the set $\A$.)
Analogous to our formalism for accommodating
a large number of possible features, we now assume that there is a
set $\F'$ of possible ``algorithm feature values'' and
a mapping $f'$ that computes the features of a given algorithm.
A performance predictor is now a map from $\F \times \F'$ to $\RR$,
taking as input the features of an algorithm $A$ and of an instance
$x$, and returning as output an estimate of $A$'s performance
on~$x$.  If $\cP$ is the set of linear predictors, for example, then
by \Propp{prop:linear} its pseudo-dimension is $d + d'$,
where $d$ and $d'$ denote the dimensions of $\F$ and $\F'$, respectively.
\end{remark}

\subsection{Application: Choosing the Step Size in Gradient
  Descent}\label{ss:gd2}

For our last PAC example, we give sample complexity results for the
problem of choosing the best step size in gradient descent.  When
gradient descent is used in practice, the step size is generally taken
much larger than the upper limits suggested by theoretical guarantees,
and often converges in many fewer iterations than with the step size
suggested by theory. This motivates the problem of learning the step
size from examples.
We view this as a baby step towards reasoning more generally about
the problem of learning good parameters for machine learning algorithms.
%
Unlike the applications we've seen so far, the parameter space here
satisfies a Lipschitz-like condition, and we can
follow the discretization approach suggested in \Rem{lip}.

\subsubsection{Gradient Descent Preliminaries}

Recall the basic gradient descent algorithm for minimizing a
function~$f$ given an initial point~$z_0$ over $\RR^n$:
\begin{enumerate}

\item Initialize $z := z_0$.

\item While $\| \grad f(z) \|_2 > \nu$: 

\begin{enumerate}

\item $z := z - \rho \cdot \grad f(z)$.

\end{enumerate}

\end{enumerate}
We take the error tolerance $\nu$ as given and focus on the more
interesting parameter, the step size~$\rho$.  Bigger values of $\rho$
have the potential to make more progress in each step, but run the
risk of overshooting a minimum of~$f$.

We instantiate the basic model (\Secc{ss:basic}) to study the
problem of learning the best step size.  There is an unknown
distribution $\dist$ over instances, where an instance $x \in \Pi$
consists of a function~$f$ and an initial point~$z_0$.  Each algorithm
$A_{\rho}$ of $\A$ is the basic gradient descent algorithm above, with
some choice~$\rho$ of a step size drawn from some fixed
interval~$[\rho_\ell,\rho_u] \subset (0,\infty)$.  The performance
measure $\cost(A,x)$ is the number of
iterations (i.e., steps) taken by the algorithm for
the instance $x$.

To obtain positive results, we need to restrict the allowable
functions~$f$ (see \Labb{Appendix}{app:gd}).
First, we assume that every function~$f$ is convex and $L$-smooth for
a known $L$. A function $f$ is {\em $L$-smooth} if it is everywhere
differentiable, and $\|\df(z_1)-\df(z_2)\| \le L\|z_1-z_2\|$
for all $z_1$ and $z_2$ (all norms in this section are the $\ell_2$
norm). Since
gradient descent is translation invariant, and $f$ is convex, we can
assume for convenience that the (uniquely attained) minimum value of
$f$ is 0, with $f(0) = 0$.

Second, we assume that the magnitudes of the initial points are
bounded, with $\|z_0\| \le Z$ for some known constant $Z > \nu$.

Third, we assume that there is a known constant $c \in (0,1)$ such
that $\|z-\rho\, \grad f(z)\| \le (1-c) \|z\|$ for all
$\rho \in [\rho_\ell,\rho_u]$.     
In other words, the norm of any point $z$ --- equivalently, the distance
to the global minimum --- decreases by some minimum factor after each
gradient descent step. We refer to this as the \emph{guaranteed progress}
condition. This is satisfied (for instance) by $L$-smooth,
$m$-strongly convex functions\footnote{A (continuously differentiable)
  function $f$ is {\em $m$-strongly convex} if $f(y) \ge f(w) + \grad
  f(w)^T(y-w) + \tfrac{m}{2} \|y-w\|^2$ for all $w,y \in \RR^n$. The
  usual notion of convexity is the same as $0$-strong convexity. Note
  that the definition of $L$-smooth implies $m \le L$.},
which is a well studied regime (see
e.g.~\cite{boyd}). The standard analysis of gradient descent implies
that $c \ge \rho m$ for $\rho \le 2/(m+L)$ over this class of functions.

Under these restrictions, we will be able to compute a nearly optimal $\rho$ given a reasonable number of samples from $\dist$.





\paragraph{Other Notation}

All norms in this section are
$\ell_2$-norms.
Unless otherwise stated, $\rho$ means $\rho$
restricted to $[\rho_\ell,\rho_u]$, and $z$ means $z$ such that $\|z\|
\le Z$. We let $g(z,\rho) := z-\rho\grad f(z)$ be the result of taking
a single gradient descent step, and $g^j(z,\rho)$ be the result of
taking $j$ gradient descent steps.

Typical textbook treatments of gradient descent assume $\rho < 2/L$ or
$\rho \le 2/(m+L)$, which give various convergence and running time
guarantees. The learning results of this section apply for any $\rho$,
but this natural threshold will still appear in our analysis and results. Let
$D(\rho) := \max\{1,L\rho-1\}$ denote how far $\rho$ is from $2/L$.

Lastly, let $H = \log(\nu/Z) / \log(1-c)$. It is easy to check that
$\cost(A_\rho, x) \le H$ for all $\rho$ and $x$.


\subsubsection{A Lipschitz-like Bound on
  \texorpdfstring{$\cost(A_\rho,x)$}{cost(Ap,x)} as a Function of
  \texorpdfstring{$\rho$}{p}.}
This will be the bulk of the argument. Our first lemma shows that for
fixed $\rho$, the gradient descent step $g$ is a Lipschitz function of
$z$, even when $\rho$ is larger than $2/L$. One might hope that the
guaranteed progress condition would be enough to show that (say) $g$
is a contraction, but the Lipschitzness of $g$ actually comes from the
$L$-smoothness. (It is not too hard to come up with non-smooth
functions that make guaranteed progress, and where $g$ is arbitrarily
non-Lipschitz.)


\begin{lemma}\label{l:boundg}
$\|g(w,\rho)-g(y,\rho)\| \le D(\rho)\|w-y\|$.
\end{lemma}

\begin{proof}
For notational simplicity, let $\a = \|w-y\|$ and $\b = \|\df(w)-\df(y)\|$.
Now,
\begin{align*}
\|g(w,\rho)-g(y,\rho)\|^2 & = \|(w-y)-\rho(\df(w)-\df(y))\|^2 \\
& = \a^2 + \rho^2\b^2 - 2\rho\la\a,\b\ra \\
& \le \a^2 + \rho^2\b^2 - 2\rho\b^2/L \\
& = \a^2 + \b^2\rho(\rho-2/L).
\end{align*}
The only inequality above is a restatement of a property of $L$-smooth
functions called the co-coercivity of the gradient, namely that
$\la\a,\b\ra \ge \b^2/L$.

Now, if $\rho \le 2/L$, then $\rho(\rho-2/L) \le 0$, and we're
done. Otherwise, $L$-smoothness implies $\b \le L\a$, so the above is
at most $\a^2(1 + L\rho(L\rho-2))$, which is the desired result.
\end{proof}

The next lemma bounds how far two gradient descent paths can drift
from each other, if they start at the same point. The main thing to
note is that the right hand side goes to 0 as $\eta$ becomes close to
$\rho$.

\begin{lemma}\label{l:boundgj}
For any $z$, $j$, and $\rho \le \eta$,
\[ \|g^j(z,\rho) - g^j(z,\eta)\| \le (\eta-\rho)\frac{D(\rho)^jLZ}{c}. \]
\end{lemma}

\begin{proof}
We first bound $\|g(w,\rho) - g(y,\eta)\|$, for any $w$ and $y$. We have
\[
g(w,\rho) - g(y,\eta) = [w - \rho\df(w)] - [y - \eta\df(y)] = g(w,\rho)- [g(y,\rho) - (\eta-\rho)\df(y)]
\]
by definition of $g$. The triangle inequality and \Lemm{l:boundg} then give
\[ \|g(w,\rho) - g(y,\eta)\| = \|g(w,\rho)-g(y,\rho) + (\eta-\rho)\df(y)\| \le D(\rho)\,\|w-y\| + (\eta-\rho)\|\df(y)\|. \]
Plugging in $w = g^j(z,\rho)$ and $y = g^j(z,\eta)$, we have
\[ \|g^{j+1}(z,\rho) - g^{j+1}(z,\eta)\| \le D(\rho)\,\|g^j(z,\rho) - g^j(z,\eta)\| + (\eta-\rho)\|\df(g^j(z,\eta))\| \]
for all $j$.

Now,
\[ \|\df(g^j(z,\eta))\| \le L \|g^j(z,\eta)\| \le L\|z\|(1-c)^j \le
LZ(1-c)^j, \]
where the first inequality is from $L$-smoothness, and the
second is from the guaranteed progress condition. Letting
$r_j = \|g^j(z,\rho) - g^j(z,\eta)\|$, we now have the simple recurrence $r_0
= 0$, and $r_{j+1} \le D(\rho)\,r_j + (\eta-\rho)LZ(1-c)^j$. One can
check via induction that
\[ r_{j+1} \le D(\rho)^j(\eta-\rho)LZ \sum_{i = 0}^j (1-c)^iD(\rho)^{-i} \]
for all $j$.
Recall that $D(\rho) \ge 1$. Rounding $D(\rho)^{-i}$ up to 1 and doing the summation gives the desired result.
\end{proof}

Finally, we show that $\cost(A_\rho,x)$ is essentially Lipschitz in
$\rho$. The ``essentially'' is necessary, since $\cost$ is
integer-valued.

\begin{lemma}\label{l:boundcost}
$|\cost(A_{\rho},x) - \cost(A_\eta,x)| \le 1$ for all $x$, $\rho$, and
  $\eta$ with $0 \le \eta-\rho \le \frac{\nu c^2}{LZ}D(\rho)^{-H}$.
\end{lemma}
\begin{proof}
Assume that $\cost(A_{\eta},x) \le
\cost(A_\rho,x)$; the argument in the other case is similar.
Let $j = \cost(A_{\eta},x)$, and recall that $j \le
H$.  By \Lemm{l:boundgj}, $\|g^j(x,\rho)- g^j(x,\eta)\| \le \nu c$.
Hence, by the triangle inequality,
$$\|g^j(x,\rho)\| \le \nu c +  \|g^j(x,\eta)\| \le \nu c + \nu.$$

Now, by the guaranteed progress condition, $\|w\| - \|g(w,p)\| \ge
c\|w\|$ for all $w$. Since we only run a gradient descent step on $w$
if $\|w\| > \nu$, each step of gradient descent run by any algorithm in $\A$
drops the magnitude of $w$ by at least $\nu c$.

Setting $w = g^j(x,\rho)$, we see that either $\|g^j(x,\rho)\| \le \nu$,
and $\cost(A_\rho, x) = j$, or that
$\|g^{j+1}(x,\rho)\| \le (\nu c + \nu) - \nu c = \nu$, and $\cost(A_\rho, x) = j+1$, as
desired.
\end{proof}

\subsubsection{Learning the Best Step Size}
We can now apply the discretization approach suggested by
\Rem{lip}. Let $K = \frac{\nu c^2}{LZ}D(\rho_u)^{-H}$. Note
that since $D$ is an increasing function, $K$ is less than or equal to
the $\frac{\nu c^2}{LZ}D(\rho)^{-H}$ of \Lemm{l:boundcost} for
every $\rho$.  Let $N$ be a minimal $K$-net, such as all integer
multiples of $K$ that lie in $[\rho_{\ell},\rho_u]$.
Note that $|N| \le \rho_u/K + 1$.

We tie everything together in the theorem
below.\footnote{Alternatively, this guarantee can be phrased in terms
  of the fat-shattering dimension (see e.g.~\cite{AB}).
In particular, $\A$ has
  1.001-fat shattering dimension at most $\log |N| = \tilde{O}(H)$.}

\begin{theorem}[Learnability of Step Size in Gradient Descent]\hspace{-.07in}
There is a learning algorithm that $(1+\eps,\delta)$-learns the optimal
algorithm in $\A$ using $m = \tilde{O}(H^3/\eps^2)$ samples from $\dist$.%
\footnote{We use $\tilde{O}(\cdot)$ to suppress logarithmic
  factors in $Z/\nu, c, L$ and $\rho_u$.}
\end{theorem}

\begin{proof}
The pseudo-dimension of $\A_N = \{A_\rho : \rho \in N\}$ is at most $\log |N|$,
since $\A_N$ is a finite set.  Since $\A_N$ is finite, it also trivially admits an ERM algorithm $L_N$,
and \Cor{ub} implies that $L_N$
$(\eps,\delta)$-learns the optimal algorithm in $\A_N$ using
$m = \tilde{O}(H^2\log |N|/\eps^2)$ samples.

Now, \Lemm{l:boundcost} implies that for every $\rho$, there is a
$\eta \in N$ such that, for every distribution $\dist$,
 the difference in expected costs of $A_\eta$
and $A_\rho$ is at most 1.
Thus $L_N$
$(1+\eps,\delta)$-learns the optimal algorithm in $\A$ using $m =
\tilde{O}(H^2\eps^{-2} \log |N|)$ samples.

Since $\log|N| = \tilde{O}(H)$, we get the desired result.
\end{proof}



\section{Online Learning of Application-Specific Algorithms}\label{s:noregret}


This section studies the problem of learning the best
application-specific algorithm {\em online}, with instances arriving
one-by-one.\footnote{The online model is obviously relevant when
training data arrives over time.  Also, even with offline data sets
that are very large, it can be computationally necessary to process
training data in a one-pass, online fashion.}
The goal is choose an algorithm at each time step, before
seeing the next instance, so that the average performance
is close to that of the best fixed algorithm in hindsight.
This contrasts with the statistical (or ``batch'') learning setup
used in \Secc{s:basic}, where the goal was to identify a single
algorithm from a batch of training instances that generalizes well to
future instances from the same distribution.
For many of the motivating examples in \Secc{s:scenarios}, both
the statistical and online learning approaches are relevant.
The distribution-free online learning formalism of this section may be
particularly appropriate when instances cannot be modeled
as i.i.d.\ draws from an unknown distribution.

\subsection{The Online Learning Model}\label{ss:online}

Our online learning model shares with the basic model of
\Secc{ss:basic} a computational or optimization
problem~$\Probl$ (e.g., MWIS), a set $\A$ of algorithms for $\Probl$
(e.g., a single-parameter family of greedy heuristics), and a
performance measure $\cost:\A \times \Probl \rightarrow [0,1]$ (e.g.,
the total weight of the returned solution).%
\footnote{One could also have $\cost$ take values in $[0,H]$ rather
than $[0,1]$, to parallel the PAC setting; we set $H=1$ here since the dependence
on $H$ will not be interesting.}
Rather than modeling the specifics of an application domain via an
unknown distribution $\dist$ over instances, however, we use an
unknown instance {\em sequence} $x_1,\ldots,x_T$.\footnote{For simplicity, we
  assume that the time horizon~$T$ is known.  This assumption can be
  removed by standard doubling techniques (e.g.~\cite{CB}).}

A learning algorithm now outputs a sequence $A_1,\ldots,A_T$ of
algorithms, rather than a single algorithm.  Each algorithm~$A_i$ is
chosen (perhaps probabilistically) with knowledge only of the previous
instances $x_1,\ldots,x_{i-1}$.  The standard goal in online
learning is to choose $A_1,\ldots,A_T$ to minimize the worst-case
(over $x_1,\ldots,x_T$) {\em regret}, defined as the average
performance loss relative to the best algorithm $A \in \A$ in
hindsight:%
\footnote{Without loss of generality, we assume $\cost$ corresponds to a maximization objective.}

\begin{equation}\label{eq:regret}
\frac{1}{T} \left(\sup_{A \in \A} \sum_{t=1}^T \cost(A,x_i)  -
\sum_{t=1}^T \cost(A_i,x_i)\right).
\end{equation}
A {\em no-regret} learning algorithm has expected (over its coin
tosses) regret~$o(1)$, as $T  \rightarrow \infty$, for every instance
sequence.
The design and analysis of no-regret online learning algorithms is a
mature field (see e.g.~\cite{CB}).
For example, many no-regret online learning algorithms are known for
the case of a finite set $|\A|$ (such as the ``multiplicative weights''
algorithm).

\subsection{An Impossibility Result for Worst-Case Instances}\label{ss:lb}


This section proves an impossibility result for no-regret online
learning algorithms for the problem of application-specific algorithm
selection.  We show this
for the running example in \Secc{ss:greedy2}:
maximum-weight independent set (MWIS) heuristics%
\footnote{Section~\ref{ss:greedy2} defined
adaptive and non-adaptive versions of the MWIS
  heuristic.  All of the
  results in \Secc{s:noregret} apply to both, so we usually won't
  distinguish between them.}
that, for
some parameter $\rho \in [0,1]$, process the vertices in order of
nonincreasing value of $w_v/(1+\deg(v))^{\rho}$.  Let $\A$ denote the
set of all such MWIS algorithms.
Since $\A$ is an infinite set,
standard no-regret results (for a finite number of
actions) do not immediately apply.  In online learning, infinite sets
of options are normally controlled through a Lipschitz condition,
stating that ``nearby'' actions always yield approximately the same
performance; our set $\A$ does not possess such a Lipschitz property
(recall \Rem{lip}).  The next section shows that these
issues are not mere technicalities --- there is enough complexity in
the set $\A$ of MWIS heuristics to preclude a no-regret learning
algorithm.

\subsubsection{A Hard Example for MWIS}
\def\on1{o_n(1)}

We show a distribution over sequences of MWIS instances for which
every (possibly randomized) algorithm has expected regret $1-\on1$.
Here and for this rest of this section, by $\on1$
we mean a function that
is independent of $T$ and tends to 0 as the number of vertices $n$
tends to infinity. Recall that $\cost(A_\rho,x)$ is the total
weight of the returned independent set, and we are trying to maximize
this quantity. The key construction is the following:

\begin{lemma}\label{l:rsexample}
For any constants $0 < r < s < 1$, there exists a MWIS instance $x$ on
at most $n$ vertices
such that $\cost(A_\rho,x) = 1$ when $\rho \in (r,s)$, and
$\cost(A_\rho,x) = \on1$ when $\rho < r$ or $\rho > s$.
\end{lemma}

\begin{proof}
Let $A, B$, and $C$ be 3 sets of vertices of sizes $m^2-2, m^3-1,$ and
$m^2+m+1$ respectively, such that their sum $m^3+2m^2+m$ is between
$n/2$ and $n$. Let $(A,B)$ be a complete bipartite graph. Let $(B,C)$ also be a
bipartite graph, with each vertex of $B$ connected to exactly one
vertex of $C$, and each vertex of $C$ connected to exactly $m-1$
vertices of $B$. See \Figg{f:badmwis}.

\begin{figure}
\begin{center}
\includegraphics[width=1in]{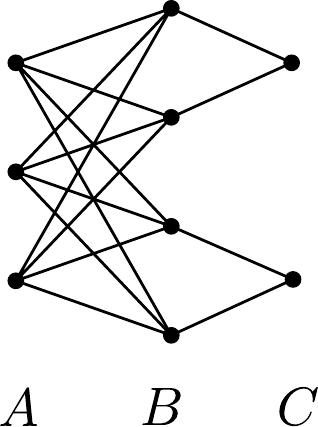}
\caption{A rough depiction of the MWIS example from \Lemm{l:rsexample}.}
\label{f:badmwis}
\end{center}
\end{figure}

Now, set the weight of every vertex in $A$, $B$, and $C$ to $tm^r$,
$t$, and $tm^{-s})$, respectively, for $t = (m^3-1)^{-1}$.
\Fig{mwiscalcs} summarizes some straightforward calculations.
We now calculate the cost of $A_\rho$ on this instance.

\begin{figure}
\begin{center}
\begin{tabular}{c|@{\hspace{.13in}}cc@{\hspace{.2in}}|@{\hspace{.17in}}ccc}
  & size & weight & deg & weight/(deg+1)$^\rho$ & size $\times$ weight \\ \hline
$A$ & $m^2-2$    & $tm^r$    & $m^3-1$   & $tm^{r-3\rho}$ & $\on1$  \\
$B$ & $m^3-1$    & $t$       & $m^2-1$   & $tm^{-2\rho}$ & 1\\
$C$ & $m^2+m+1$  & $tm^{-s}$  & $m-1$     & $tm^{-s-\rho}$ & $\on1$ \\
\end{tabular}
\caption{Details and simple calculations for the vertex sets comprising the MWIS example from \Lemm{l:rsexample}.}
\label{fig:mwiscalcs}
\end{center}
\end{figure}

If $\rho < r$, the algorithm $A_\rho$ first chooses a vertex in $A$, which
immediately removes all of $B$, leaving at most $A$ and $C$ in the
independent set. The total weight of $A$ and $C$ is $\on1$, so
$\cost(A_\rho)$ is $\on1$.

If $\rho > s$, the algorithm first chooses a vertex in $C$, which
removes a small chunk of $B$. In the non-adaptive setting, $A_\rho$
simply continues choosing vertices of $C$ until $B$ is gone. In the
adaptive setting, the degrees of the remaining elements of $B$ never
change, but the degrees of $A$ decrease as we pick more and more
elements of $C$. We eventually pick a vertex of $A$, which immediately
removes the rest of $B$. In either case, the returned independent set
has no elements from $B$, and hence has $\cost$ $\on1$.

If $\rho \in (r,s)$, the algorithm first picks a vertex of $B$,
immediately removing all of $A$, and one element of $C$. The remaining
graph comprises $m-2$ isolated vertices of $B$ (which get added to the
independent set), and $m^2+m$ stars with centers in $C$ and leaves in
$B$. It is easy to see that both the adaptive and the non-adaptive
versions of the heuristic return exactly $B$.
\end{proof}

We are now ready to state the main result of this section.
\begin{theorem}[Impossibility of Worst-Case Online Learning]\label{t:lb}
There is a distribution on MWIS input sequences over which every
algorithm has expected regret $1-\on1$.
\end{theorem}

\begin{proof}
Let $t_j = (r_j,s_j)$ be a distribution over sequences of
nested intervals with $s_j -r_j = n^{-j}$,
$t_0 = (0,1)$, and with $t_j$ chosen uniformly at random from
within $t_{j-1}$. Let $x_j$ be an MWIS instance on up to $n$ vertices such that
$\cost(A_\rho,x) = 1$ for $\rho \in (r_j,s_j)$, and
$\cost(A_\rho,x) = \on1$ for $\rho < r_j$ and $\rho > s_j$ (\Lemm{l:rsexample}).

The adversary presents the instances $x_1,x_2, \dots,x_T$, in that
order. For every $\rho \in t_T$, $\cost(A_\rho,x_j) = 1$ for all
$j$. However, at every step $t$, no algorithm can have a better than
$1/n$ chance of picking a $\rho_t$ for which $\cost(A_{\rho_t},x_t) =
\Omega(1)$, even given $x_1,x_2,\dots,x_{t-1}$ and full knowledge of
how the sequence is generated.
\end{proof}

\subsection{A Smoothed Analysis}\label{ss:ub}
\def\pn{\poly(n^{-1})}

Despite the negative result above, we can show a ``low-regret''
learning algorithm for MWIS under a slight restriction on how the
instances $x_t$ are chosen. By low-regret we mean that the regret can
be made polynomially small
as a function of the number of vertices $n$.
This is not the same as the no-regret condition, which requires
regret tending to 0 as $T \rightarrow \infty$.
Nevertheless, inverse polynomially small regret $\pn$
is a huge improvement
over the constant regret incurred in the worst-case lower bound
(Theorem~\ref{t:lb}).

We take the approach suggested by smoothed analysis~\cite{ST09}.
%
Fix a parameter $\sigma \in (0,1)$. We allow each MWIS instance $x_t$ to
have an
arbitrary graph on $n$ vertices, but we replace each vertex weight $w_v$ with
a probability distribution $\Delta_{t,v}$ with density
at most $\sigma^{-1}$ (pointwise) and support in $[0,1]$. A simple
example of such a
  distribution with $\sigma = 0.1$ is the uniform distribution
  on $[0.6, 0.65] \cup [0.82, 0.87]$.
To instantiate the instance $x_t$, we draw each vertex weight
  from its distribution $\Delta_{t,v}$.  We call such an instance a
  {\em $\sigma$-smooth MWIS instance}.

For small $\sigma$, this is quite a weak restriction.
As $\sigma \rightarrow 0$ we return to the worst-case setting,
and \Thmm{t:lb} can be extended to the case of $\sigma$ exponentially
small in $n$.
Here, we think of $\sigma$ as bounded below by an
(arbitrarily small) inverse polynomial function of $n$.
One example of such a smoothing
is to start with an arbitrary MWIS instance, keep the
first $O(\log n)$ bits of every weight, and set the remaining
lower-order bits at random.

The main result of this section is a polynomial-time low-regret learning
algorithm for sequences of $\sigma$-smooth MWIS instances. Our
strategy is to take a finite net $N \subset [0,1]$ such
that, for every algorithm $A_\rho$ and smoothed instance $x_t$,
with high probability over $x_t$ the performance of $A_{\rho}$ is
identical to that of some algorithm in
$\{A_\eta : \eta \in N\}$.
We can then use any
off-the-shelf no-regret algorithm to output a sequence of algorithms
from the finite set $\{A_\eta : \eta \in N\}$, and show the desired
regret bound.

\subsubsection{A Low-Regret Algorithm for \texorpdfstring{$\sigma$}{sigma}-Smooth MWIS}

We start with some definitions. For a fixed $x$, let $\t'(x)$ be the
set of \emph{transition points}, namely,%
\footnote{The corner cases $\rho = 0$ and $\rho = 1$ require straightforward but
  wordy special handling in this statement and in several others in this
  section. We omit these details to keep the argument free of clutter.}
\[ \tau'(x) := \{\rho : A_{\rho-\omega}(x) \ne A_{\rho+\omega}(x) \mbox{ for
  arbitrarily small } \omega\}. \]
It is easy to see $\t'(x) \subset \tau(x)$, where
\[ \tau(x) := \{\rho : w_{v_1}/k_1^\rho = w_{v_2}/k_2^\rho \mbox{
  for some } v_1,v_2,k_1,k_2\in [n]; \, k_1,k_2 \ge 2\}. \]
With probability 1, the vertex weights
$w_v$ are all distinct and non-zero, so we can rewrite
$\tau$ as
\[ \tau(x) := \left\{\rho(
v_1,v_2,k_1,k_2)  \
:\
v_1,v_2,k_1,k_2 \in [n];\, k_1,k_2 \ge 2;\, k_1 \ne k_2\right\}, \]
where
\begin{equation}\label{eq:rho}
\rho(v_1,v_2,k_1,k_2) = \frac{\ln(w_{v_1}) -
  \ln(w_{v_2})}{\ln(k_1) - \ln(k_2)}
\end{equation}
and $\ln$ is the natural logarithm function.
The main technical task is to show that no two elements of
$\tau(x_1) \cup \cdots \cup \tau(x_m)$ are within $q$ of each other, for a
sufficiently large $q$ and sufficiently large $m$, and with high enough
probability over the randomness in the weights of the $x_t$'s.

We first make a few straightforward computations.
The following brings the noise into log space.
\begin{lemma}\label{l:logdensity}
If $X$ is a random variable over $(0,1]$ with density at most $\delta$, then $\ln(X)$ also has density at most $\delta$.
\end{lemma}
\begin{proof}
  Let $Y = \ln(X)$, let $f(x)$ be the density of $X$ at $x$, and let $g(y)$ be the density of $Y$ at $y$. Note that $X = e^Y$, and let $v(y) = e^y$. Then $g(y) = f(v(y)) \cdot v'(y) \le f(v(y)) \le \delta$ for all $y$.
\end{proof}

Since $|\ln(k_1) - \ln(k_2)| \le \ln n$, \Lemm{l:logdensity} and our definition of $\sigma$-smoothness implies the following.
\begin{corollary}\label{c:rhodensity}
For every $\sigma$-smooth MWIS instance $x$, and every $v_1,v_2,k_1,k_2
\in [n], \, k_1,k_2 \ge 2, \, k_1 \ne k_2$, the density of
$\rho(v_1,v_2,k_1,k_2)$ is bounded by $\sigma^{-1}\ln n$.
\end{corollary}

We now show that it is unlikely that two distinct elements of
$\tau(x_1) \cup \cdots \cup \tau(x_m)$ are very close to each other.

\begin{lemma}\label{lem:faraway}
Let $x_1,\dots, x_m$ be $\sigma$-smooth MWIS instances. The
probability that no two distinct elements of $\tau(x_1) \cup \cdots \cup
\tau(x_m)$ are within $q$ of each other is at least $1-4q
\sigma^{-1}m^2n^8\ln n$.
\end{lemma}

\begin{proof}
Fix instances $x$ and $x'$, and choices of
$(v_1,v_2,k_1,k_2)$ and $(v'_1,v'_2,k'_1,k'_2)$.
Denote by $\rho$ and $\rho'$ the corresponding random variables,
defined as in~\eqref{eq:rho}.
%
We compute the probability that $|\rho-\rho'| \le q$ under various
scenarios, over the randomness in the vertex weights.
We can ignore the case where $x=x'$, $v_1=v_1'$, $v_2=v_2'$, and
$k_1/k_2 = k'_1/k'_2$, since then $\rho=\rho'$ with probability~1.
We consider three other cases.

\vspace{.75\baselineskip}
\noindent
\textbf{Case 1:}
Suppose $x \ne x'$, and/or $\{v_1,v_2\}$ and $\{v'_1, v'_2\}$
don't intersect. In this case, $\rho$ and $\rho'$ are independent
random variables. Hence the maximum density of $\rho-\rho'$ is at most
the maximum density of $\rho$, which is $\sigma^{-1}\ln n$ by
\Labb{Corollary}{c:rhodensity}. The probability that $|\rho-\rho'| \le
q$ is hence at most $2q\cdot \sigma^{-1}\ln n$.

\vspace{.75\baselineskip}
\noindent
\textbf{Case 2:}
Suppose $x=x'$, and $\{v_1, v_2\}$ and $\{v'_1, v'_2\}$ share
exactly one element, say
$v_2 = v'_2$.
Then $\rho-\rho'$
has the form $X-Y$, where $X =
\frac{\ln(w_{v_1})}{\ln(k_1) - \ln(k_2)}$ and $X$ and $Y$ are
independent. Since the maximum density of $X$ is at most
$\sigma^{-1}\ln
n$ (by \Lemm{l:logdensity}), the probability that $|\rho-\rho'| \le q$
is again at most $2q\cdot \sigma^{-1}\ln n$.

\vspace{.75\baselineskip}
\noindent
\textbf{Case 3:} Suppose $x=x'$ and $\{v_1,v_2\} = \{v'_1, v'_2\}$.
In this case, $k_1/k_2 \neq k'_1/k'_2$.
Then
\begin{align*}
 |\rho - \rho'|
& = \left| \big(\ln(w_{v_1}) - \ln(w_{v_2})\big)\left(\frac{1}{\ln(k_1)-\ln(k_2)} - \frac{1}{\ln(k'_1) - \ln(k'_2)}\right) \right| \\
& \ge \frac{|\ln(w_{v_1}) - \ln(w_{v_2})|}{n^2}.
\end{align*}
Since $w_{v_1}$ and $w_{v_2}$ are independent, the maximum
density of the right hand side is at most $\sigma^{-1}n^2$, and hence
the probability that $|\rho-\rho'| \le q$ is at most
$2q\cdot\sigma^{-1}n^2$.

We now upper bound the number of tuple pairs
that can appear in each case above.
Each set $\tau(x_i)$ has at most $n^4$ elements, so there are at
most
$m^2n^8$ pairs in Cases 1
and 2. There are at most
$n^4$
choices of $(k_1,
k_2, k'_1, k'_2)$ for each $(x, v_1, v_2)$ in Case 3, for a total of
at most $mn^6$ pairs.
The theorem now follows from the union bound.
\end{proof}

Lastly, we formally state the existence of no-regret algorithms
for the case of finite $|\A|$.
\begin{fact}[E.g.~\cite{LW94}]\label{f:mwexist}
For a finite set of algorithms $\A$, there exists a randomized online
learning algorithm $L^*$ that, for every $m > 0$, has expected
regret at most $O(\sqrt{(\log |\A|)/m})$ after seeing $m$ instances.
If the
time cost of evaluating $\cost(A,x)$ is bounded by $B$, then this
algorithm runs in $O(B|\A|)$ time per instance.
\end{fact}

We can now state our main theorem.
\begin{theorem}[Online Learning of Smooth MWIS]
There is an online learning algorithm
for $\sigma$-smooth MWIS
that runs in time
$\poly(n, \sigma^{-1})$ and has expected regret
at most $\pn$
(as $T \rightarrow \infty$).
\end{theorem}

\begin{proof}
Fix a sufficiently large constant $d > 0$ and
consider the first $m$ instances of our sequence, $x_1,\dots,x_m$,
with $m = n^d\,\ln (\sigma^{-1})$.  Let
$q = 1/(n^d \cdot 4\sigma^{-1}m^2n^8\ln n)$. Let $E_q$ be the
event that every two distinct elements of $\tau(x_1) \cup \cdots \cup
\tau(x_m)$ are at least $q$ away from each other. By \Lem{faraway},
$E_q$ holds with probability at least $1-1/n^d$ over the randomness
in the vertex weights.

Now, let $\A_N = \{A_{i} : i \in \{0,q,2q,\dots, \lfloor 1/q\rfloor q,
1\}\}$ be a ``$q$-net.'' Our desired algorithm $L$ is simply the algorithm
$L^*$ from \Labb{Fact}{f:mwexist}, applied to $\A_N$. We now
analyze its expected regret.

If $E_q$ does hold, then for every algorithm $A\in \A$, there is an
algorithm $A'\in \A_N$
such that $\cost(A,x_t) = \cost(A',x_t)$ for $x_1,\dots,x_m$. In other
words, the best algorithm of $\A_N$ is no worse than the best
algorithm from all of $\A$, and in this case
the expected regret of $L$ is simply that of $L^*$.
By \Labb{Fact}{f:mwexist} and our choice of~$m$,
the expected regret (over the coin flips made by $L^*$)
is at most inverse polynomial in~$n$.

If $E_q$ does not hold, our regret is at most 1, since $\cost$ is
between 0 and 1.
Averaging over the cases where~$E_q$ does and does not hold (with
probabilities $1-1/n^d$ and $1/n^d$),
the expected regret of the learning algorithm~$L$ (over the randomness in
$L^*$ and in the instances) is at most inverse polynomial in~$n$.
%
\end{proof}

\section{Conclusions and Future Directions}\label{s:conc}

Empirical work on application-specific algorithm selection has far
outpaced theoretical analysis of the problem, and this paper takes
an initial step towards redressing this imbalance.
We formulated the problem as one of learning the best
algorithm or algorithm sequence from a class with respect to an
unknown input distribution or input sequence.
Many state-of-the-art empirical approaches to algorithm selection map
naturally to instances of our learning frameworks.
This paper demonstrates that many well-studied classes of algorithms have
small pseudo-dimension, and thus it is possible to learn a
near-optimal algorithm from a relatively modest amount of data.
While worst-case guarantees for no-regret
online learning algorithms are impossible, good online
learning algorithms exist in a natural smoothed model.

Our work suggests numerous wide-open research directions worthy of
further study.  For example:
\begin{enumerate}

\item Which computational problems admit a class
of algorithms that simultaneously has low representation error and
small pseudo-dimension (like in \Secc{ss:sorting2})?

\item Which algorithm classes can be learned online, in either a
  worst-case or a smoothed model?

\item
When is it possible to learn a near-optimal algorithm using only a
polynomial amount of computation, ideally with
a learning algorithm that is better than
brute-force search?  Alternatively, are there (conditional) lower
bounds stating that brute-force search is necessary for
learning?\footnote{Recall the discussion in \Secc{ss:gd}: even
  in practice, the state-of-the-art for application-specific algorithm
  selection often boils down to brute-force search.}


\item
Are there any non-trivial relationships between statistical
learning measures of the complexity of an algorithm class
and more traditional computational complexity measures?

\item
How should instance features be chosen to minimize the representation
error of the induced family of algorithm selection maps (cf.,
\Secc{ss:features})?

\end{enumerate}

\subsection*{Acknowledgements}

We are grateful for the many helpful
comments provided by the anonymous SICOMP and ITCS reviewers.

\bibliographystyle{abbrv}
\bibliography{features}

\begin{thebibliography}{10}

\bibitem{sesh}
N.~Ailon, B.~Chazelle, S.~Comandur, and D.~Liu.
\newblock Self-improving algorithms.
\newblock In {\em Proceedings of the Symposium on Discrete Algorithms
  {(SODA)}}, pages 261--270, 2006.

\bibitem{AB}
M.~Anthony and P.~L. Bartlett.
\newblock {\em Neural Network Learning: Theoretical Foundations}.
\newblock Cambridge University Press, 1999.

\bibitem{A+01}
V.~Arya, N.~Garg, R.~Khandekar, A.~Meyerson, K.~Munagala, and V.~Pandit.
\newblock Local search heuristics for $k$-median and facility location
  problems.
\newblock {\em SIAM Journal on Computing}, 33(3):544--562, 2004.

\bibitem{BB12}
J.~Bergstra and Y.~Bengio.
\newblock Random search for hyper-parameter optimization.
\newblock {\em Journal of Machine Learning Research}, 13(1):281--305, 2012.

\bibitem{borodin}
A.~Borodin, M.~N. Nielsen, and C.~Rackoff.
\newblock ({I}ncremental) priority algorithms.
\newblock {\em Algorithmica}, 37(4):295--326, 2003.

\bibitem{boyd}
S.~Boyd and L.~Vandenberghe.
\newblock {\em Convex optimization}.
\newblock Cambridge University Press, 2004.

\bibitem{fcc}
{U. S. Congressional Budget Office}: The budget and economic outlook: 2015 to
  2025.
\newblock 2014.

\bibitem{CB}
N.~Cesa{-}Bianchi and G.~Lugosi.
\newblock {\em Prediction, learning, and games}.
\newblock Cambridge University Press, 2006.

\bibitem{CMS10}
K.~L. Clarkson, W.~Mulzer, and C.~Seshadhri.
\newblock Self-improving algorithms for convex hulls.
\newblock In {\em Proceedings of the Symposium on Discrete Algorithms (SODA)},
  pages 1546--1565, 2010.

\bibitem{CMS12}
K.~L. Clarkson, W.~Mulzer, and C.~Seshadhri.
\newblock Self-improving algorithms for coordinate-wise maxima.
\newblock In {\em Proceedings of the Symposium on Computational Geometry
  (SoCG)}, pages 277--286, 2012.

\bibitem{CS08}
K.~L. Clarkson and C.~Seshadhri.
\newblock Self-improving algorithms for {D}elaunay triangulations.
\newblock In {\em Proceedings of the Symposium on Computational Geometry
  {(SoCG)}}, pages 148--155, 2008.

\bibitem{devroye}
L.~Devroye.
\newblock {\em Lectures Notes on Bucket Algorithms}.
\newblock Birkh{\"a}user, 1986.

\bibitem{fink}
E.~Fink.
\newblock How to solve it automatically: Selection among problem solving
  methods.
\newblock In {\em Proceedings of the International Conference on Artificial
  Intelligence Planning Systems}, pages 128--136, 1998.

\bibitem{ag}
A.~Gathmann.
\newblock Lectures notes on algebraic geometry.
\newblock 2014.

\bibitem{haussler}
D.~Haussler.
\newblock Decision theoretic generalizations of the {PAC} model for neural net
  and other learning applications.
\newblock {\em Information and Computation}, 100(1):78--150, 1992.

\bibitem{horvitz}
E.~Horvitz, Y.~Ruan, C.~P. Gomes, H.~A. Kautz, B.~Selman, and D.~M. Chickering.
\newblock A {B}ayesian approach to tackling hard computational problems.
\newblock In {\em Proceedings of the Conference in Uncertainty in Artificial
  Intelligence (UAI)}, pages 235--244, 2001.

\bibitem{huang}
L.~Huang, J.~Jia, B.~Yu, B.~Chun, P.~Maniatis, and M.~Naik.
\newblock Predicting execution time of computer programs using sparse
  polynomial regression.
\newblock In {\em Proceedings of Advances in Neural Information Processing
  Systems (NIPS)}, pages 883--891, 2010.

\bibitem{H+14}
F.~Hutter, L.~Xu, H.~H. Hoos, and K.~Leyton{-}Brown.
\newblock Algorithm runtime prediction: Methods {\&} evaluation.
\newblock {\em Artificial Intelligence}, 206:79--111, 2014.

\bibitem{J06}
G.~J.~O. Jameson.
\newblock Counting zeros of generalized polynomials: {D}escartes’ rule of
  signs and {L}aguerre’s extensions.
\newblock {\em Mathematical Gazette}, 90(518):223--234, 2006.

\bibitem{JM97}
D.~S. Johnson and L.~A. McGeoch.
\newblock The traveling salesman problem: A case study in local optimization.
\newblock In E.~Aarts and J.~K. Lenstra, editors, {\em Local Search in
  Combinatorial Optimization}, pages 215--310. Wiley, 1997.
\newblock Reprinted by Princeton University Press, 2003.

\bibitem{knuth}
D.~E. Knuth.
\newblock Estimating the efficiency of backtrack programs.
\newblock {\em Mathematics of Computation}, 29:121--136, 1975.

\bibitem{kotthoff}
L.~Kotthoff, I.~P. Gent, and I.~Miguel.
\newblock An evaluation of machine learning in algorithm selection for search
  problems.
\newblock {\em {AI} Communications}, 25(3):257--270, 2012.

\bibitem{LOS02}
D.~Lehmann, L.~I. {O'Callaghan}, and Y.~Shoham.
\newblock Truth revelation in approximately efficient combinatorial auctions.
\newblock {\em Journal of the ACM}, 49(5):577--602, 2002.

\bibitem{LNS09}
K.~Leyton{-}Brown, E.~Nudelman, and Y.~Shoham.
\newblock Empirical hardness models: Methodology and a case study on
  combinatorial auctions.
\newblock {\em Journal of the {ACM}}, 56(4), 2009.

\bibitem{LW94}
N.~Littlestone and M.~K. Warmuth.
\newblock The weighted majority algorithm.
\newblock {\em Information and computation}, 108(2):212--261, 1994.

\bibitem{long}
P.~M. Long.
\newblock Using the pseudo-dimension to analyze approximation algorithms for
  integer programming.
\newblock In {\em Proceedings of the International Workshop on Algorithms and
  Data Structures (WADS)}, pages 26--37, 2001.

\bibitem{MS14}
P.~Milgrom and I.~Segal.
\newblock Deferred-acceptance auctions and radio spectrum reallocation.
\newblock Working paper, 2014.

\bibitem{MM14}
M.~Mohri and A.~M. Medina.
\newblock Learning theory and algorithms for revenue optimization in second
  price auctions with reserve.
\newblock In {\em Proceedings of the International Conference on Machine
  Learning (ICML)}, pages 262--270, 2014.

\bibitem{MR15}
J.~Morgenstern and T.~Roughgarden.
\newblock The pseudo-dimension of near-optimal auctions.
\newblock In {\em Proceedings of Advances in Neural Information Processing
  Systems}, 2015.

\bibitem{STY03}
S.~Sakai, M.~Togasaki, and K.~Yamazaki.
\newblock A note on greedy algorithms for the maximum weighted independent set
  problem.
\newblock {\em Discrete Applied Mathematics}, 126(2):313--322, 2003.

\bibitem{ST09}
D.~A. Spielman and S.~Teng.
\newblock Smoothed analysis: an attempt to explain the behavior of algorithms
  in practice.
\newblock {\em Communications of the ACM}, 52(10):76--84, 2009.

\bibitem{SB06}
N.~Srebro and S.~Ben-David.
\newblock Learning bounds for support vector machines with learned kernels.
\newblock In {\em Proceedings of the 19th Annual Conference on Learning
  Theory}, pages 169--183, 2006.

\bibitem{xu}
L.~Xu, F.~Hutter, H.~H. Hoos, and K.~Leyton{-}Brown.
\newblock {SATzilla}: Portfolio-based algorithm selection for {SAT}.
\newblock {\em J. Artificial Intelligence Research {(JAIR)}}, 32:565--606,
  2008.

\bibitem{xu11}
L.~Xu, F.~Hutter, H.~H. Hoos, and K.~Leyton{-}Brown.
\newblock {Hydra-MIP}: Automated algorithm configuration and selection for
  mixed integer programming.
\newblock In {\em Proceedings of the RCRA workshop on combinatorial explosion
  at the International Joint Conference on Artificial Intelligence (IJCAI)},
  pages 16--30, 2011.

\bibitem{xu2}
L.~Xu, F.~Hutter, H.~H. Hoos, and K.~Leyton{-}Brown.
\newblock {SATzilla2012}: Improved algorithm selection based on cost-sensitive
  classification models.
\newblock In {\em Proceedings of the International Conference on Theory and
  Applications of Satisfiability Testing (SAT)}, 2012.

\end{thebibliography}

\appendix

\section{A Bad Example for Gradient Descent}\label{app:gd}

We depict a family $\mathcal{F}$ of real-valued functions
(defined on the plane $\RR^2$)
for which the class $\A$
of gradient descent algorithms from Section~\ref{ss:gd2} has
infinite pseudo-dimension.
We parameterize each function $f_I \in \mathcal{F}$
in the class by a finite subset $I\subset
[0,1]$.  The ``aerial view'' of $f_I$ is as follows.


\medskip
\begin{center}
\includegraphics[width=4in]{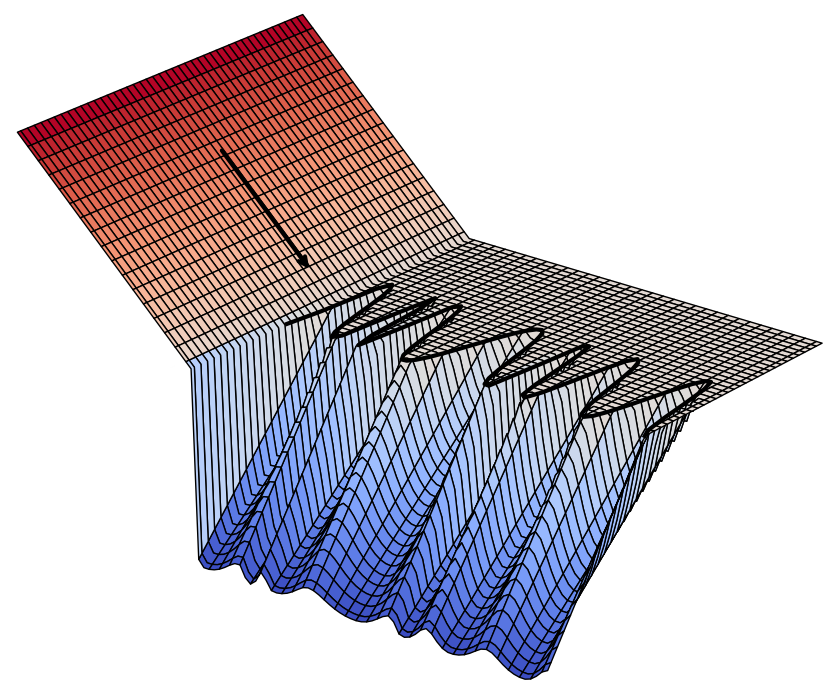}
\end{center}
\medskip

The ``squiggle'' $s(I)$ intersects the relevant axis at exactly $I$
(to be concrete, let $s(I)$ be the monic polynomial with roots at
$I$). We fix the initial point $z_0$ to be at the tail of the arrow
for all instances, and fix $\rho_\ell$ and $\rho_u$ so that the first
step of gradient descent takes $z_0$ from the red incline into the
middle of the black and blue area. Let $x_I$ be the instance
corresponding to $f_I$ with starting point $z_0$. If for a certain
$\rho$ and $I$, $g(z_0,\rho)$ lands in the flat, black area, gradient
descent stops immediately and $\cost(A_\rho, x_I) = 1$. If
$g(z_0,\rho)$ instead lands in the sloped, blue area, $\cost(A_\rho,
x_I) \gg 1$.


It should be clear that $\mathcal{F}$ can shatter any finite subset of
$(\rho_\ell,\rho_u)$, and hence has infinite pseudo-dimension. One can
also make slight modifications to ensure that all the functions in
$\mathcal{F}$ are continuously differentiable and $L$-smooth.

\typeout{Get arXiv to do 4 passes: Label(s) may have changed. Rerun}
\end{document}